\documentclass[accepted]{uai2026} % after acceptance, for a revised version; 
\usepackage[american]{babel}
\usepackage{amsmath}
\usepackage{amsfonts}
\usepackage{amssymb}
\usepackage{amsxtra}
\usepackage{amsthm}
\usepackage{mathrsfs}
\usepackage{color}
\usepackage[table]{xcolor}
\usepackage{algorithm,algorithmic}
\usepackage{booktabs}
\usepackage{tabularx}
\usepackage{multirow}
\usepackage{hyperref}
\usepackage{ragged2e}
\usepackage{bbm}
\newcolumntype{Y}{>{\RaggedRight\arraybackslash}X}

\usepackage{natbib} % has a nice set of citation styles and commands
\usepackage{mathtools} % amsmath with fixes and additions
\usepackage{booktabs} % commands to create good-looking tables
\usepackage{tikz} % nice language for creating drawings and diagrams

\newtheorem{theorem}{Theorem}[section]
\newtheorem{proposition}[theorem]{Proposition}
\newtheorem{lemma}[theorem]{Lemma}
\newtheorem{corollary}[theorem]{Corollary}
\theoremstyle{definition}

\newtheorem{assumption}[theorem]{Assumption}
\theoremstyle{remark}
\newtheorem{remark}[theorem]{Remark}
\definecolor{highlight}{RGB}{235, 240, 255}
\DeclareMathOperator*{\argmax}{arg\,max}

\title{Online Learning for Multi-Layer Hierarchical Inference \\ under Partial and Policy-Dependent Feedback
}

\author[1]{\href{mailto:<haoranz@austin.utexas.edu>?Subject=Online Learning for Multi-Layer Hierarchical Inference under Partial and Policy-Dependent Feedback}{Haoran Zhang}}
\author[1]{Seohyeon Cha}
\author[1]{Hasan Burhan Beytur}
\author[2]{Kevin S Chan}
\author[1]{Gustavo de Veciana}
\author[1]{Haris Vikalo}
\affil[1]{%
    Department of Electrical and Computer Engineering\\
    The University of Texas at Austin
}
\affil[2]{%
    DEVCOM Army Research Laboratory
}
  
\begin{document}
\maketitle

% Time scale for offloading and onloading

\begin{abstract}
% Hierarchical inference systems route tasks across multiple computational layers, where each node either finalizes a prediction locally or offloads the task to a node in the next layer for further processing. Learning optimal routing policies in such systems is challenging: inference loss is defined recursively across layers, while feedback on prediction error is revealed only when tasks reach a terminal oracle layer. This induces a partial-feedback learning problem in which observability probabilities depend recursively on downstream routing decisions, causing importance-weighted loss estimators to suffer from depth-amplified variance. We study online routing for multi-layer hierarchical inference under long-term resource constraints and terminal-only feedback. We formalize the recursive loss structure and show that naive importance-weighted contextual bandit methods become unstable as feedback probability decays along the hierarchy. To address this, we develop a variance-reduced EXP4-based algorithm integrated with Lyapunov optimization, yielding unbiased loss estimation and stable learning under sparse and policy-dependent feedback. We provide regret guarantees relative to the best fixed routing policy in hindsight and establish near-optimality under stochastic arrivals and resource constraints. Experiments on large-scale multi-task workloads demonstrate improved stability and performance compared to standard importance-weighted approaches, particularly in deep hierarchical systems.
Hierarchical inference systems route tasks across multiple computational layers, where each node may either finalize a prediction locally or offload the task to a node in the next layer for further processing. Learning optimal routing policies in such systems is challenging: inference loss is defined recursively across layers, while feedback on prediction error is revealed only at a terminal oracle layer. 
This induces a partial, policy-dependent feedback structure in which observability probabilities decay with depth, causing importance-weighted estimators to suffer from amplified variance.
We study online routing for multi-layer hierarchical inference under long-term resource constraints and terminal-only feedback. We formalize the recursive loss structure and show that naive importance-weighted contextual bandit methods become unstable as feedback probability decays along the hierarchy. To address this, we develop a variance-reduced EXP4-based algorithm integrated with Lyapunov optimization, yielding unbiased loss estimation and stable learning under sparse and policy-dependent feedback. We provide regret guarantees relative to the best fixed routing policy in hindsight and establish near-optimality under stochastic arrivals and resource constraints. Experiments on large-scale multi-task workloads demonstrate improved stability and performance compared to standard importance-weighted approaches.
% , particularly in deep hierarchical systems.
\end{abstract}

\section{Introduction}\label{sec:intro}

Large language models (LLMs) and related foundation models have enabled a broad range of tasks, from text generation to multimodal reasoning~\citep{achiam2023gpt,wolf2019huggingface,chang2024survey}. In practice, these models exhibit varied strengths -- lightweight models are efficient but less accurate on complex tasks, whereas larger models provide stronger performance at substantially higher computational and communication cost. As a result, selecting and deploying appropriate models for diverse workloads has become a major systems and learning challenge~\citep{mei2025omnirouter}.

A key question in this setting is where to execute inference tasks. Cloud-based inference achieves high accuracy but incurs latency, bandwidth, and monetary cost~\citep{friha2024llm}. Deploying models at the edge reduces communication overhead but introduces potentially severe memory and computation constraints. Although quantization and compression techniques help alleviate some of these limitations~\citep{frantar2022gptq,li2025gptaq,cha2026regularized}, lightweight edge models often struggle to perform on complex inputs. This trade-off motivates hierarchical inference (HI) architectures~\citep{yin2018hierarchical}, in which tasks are first processed locally and selectively offloaded to progressively more capable upstream nodes when necessary.

While HI improves resource utilization by processing easy tasks locally and reserving costly upstream computation for harder cases, learning optimal routing policies in such systems is challenging: 
(1) Decisions are sequential across layers, and the inference loss is defined recursively along the routing path. 
(2) Feedback on prediction error is typically revealed only when tasks reach a final oracle layer (e.g., cloud verification or human judgment), yielding sparse, terminal-only observations.
This induces a partial-feedback learning problem where the probability that a task's loss is observed depends recursively on subsequent routing decisions, coupling policies across layers and amplifying the variance of standard importance-weighted estimators. Existing HI methods typically focus on shallow architectures, static optimization, and/or single-destination offloading~\citep{beytur2024optimization,yuan2025local,cha2025batching,guo2020online}, and do not address this recursive, policy-dependent feedback structure in dynamic multi-layer systems.

\begin{figure}
    \centering
    \includegraphics[width=0.95\linewidth]{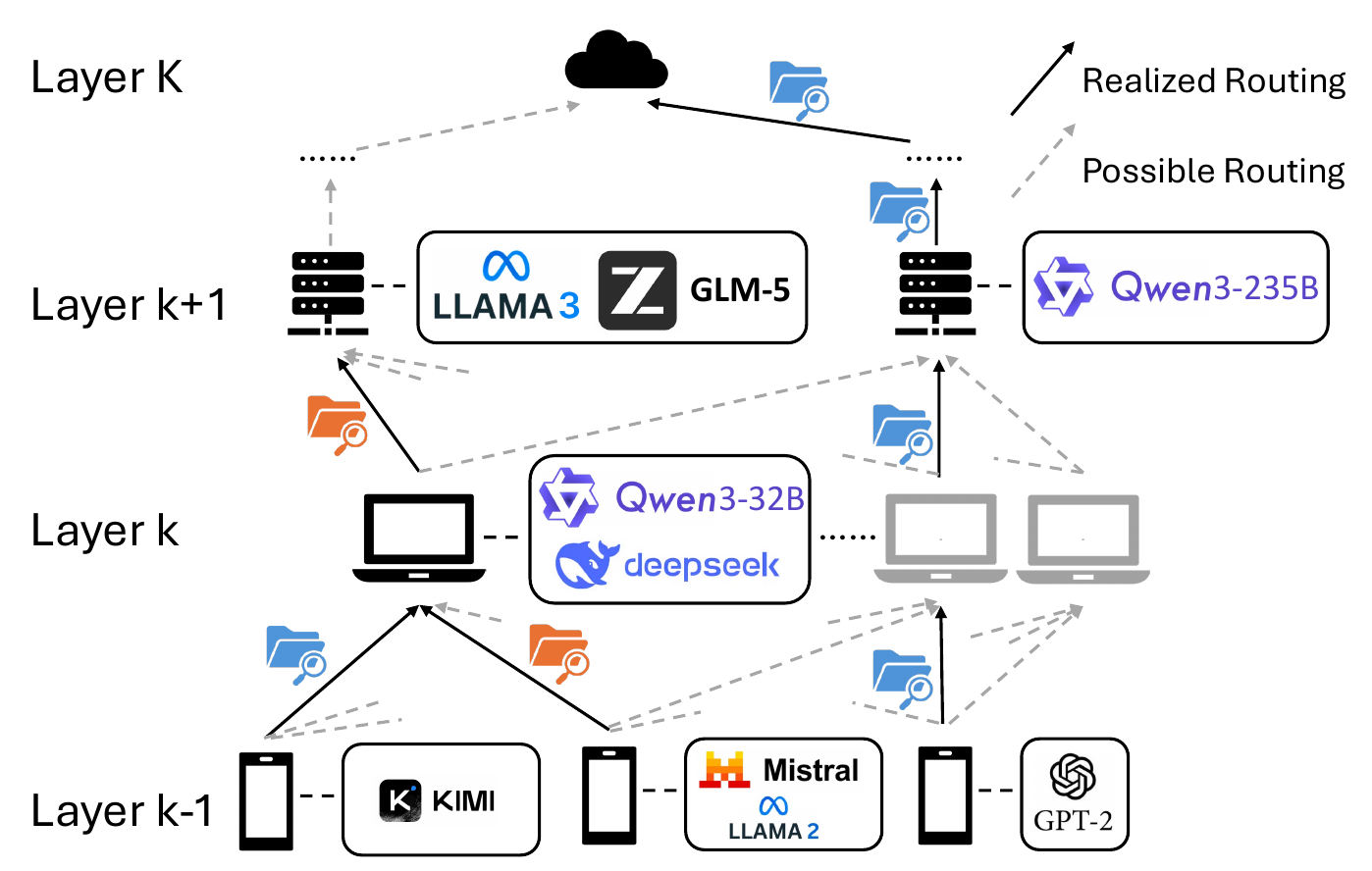}
    \caption{Hierarchical inference with multi-destination offloading. Routing decisions couple expected inference loss with upstream resource consumption. Prediction error is observed only at the terminal layer, resulting in policy-dependent feedback.}
    \label{fig:overview}
    \vspace{-0.15in}
\end{figure}

In this work (Figure \ref{fig:overview}), we study online routing in multi-layer HI systems that allow multi-destination offloading and dynamic model placement. While such systems improve utilization of resources, learning optimal routing policies is challenging due to the structure of the induced learning problem. First, routing decisions are sequentially coupled across layers, as the loss incurred by a task depends recursively on downstream decisions along its routing path. Second, feedback on inference error is revealed only if and when a task reaches the final (oracle or human judge) layer. As a result, the probability that a loss for a task is observed depends recursively on downstream routing policies, i.e., observability is policy-dependent and depth-sensitive. Third, routing policies must satisfy long-term resource constraints, and are affected by model placement decisions, further coupling learning and computation/communication across layers. Together, these characteristics define a partial-feedback online learning problem that is not addressed by the existing shallow or single-destination formulations of the problem.

To address these challenges, we formalize routing in HI as an online learning problem under recursive loss and partial (terminal-only) feedback under long-term resource constraints. We develop a distributed algorithm that integrates Lyapunov optimization for enforcing constraints with a variance-reduced EXP4-based contextual bandit method for routing. The main technical component of the proposed framework is an unbiased, variance-reduced loss estimator designed for online learning under a policy-dependent feedback structure induced by hierarchical routing. This enables stable policy learning despite depth-dependent feedback probabilities due to terminal-only observations. In addition, we incorporate periodic model placement updates under memory constraints, enabling routing and resource allocation to adapt to stochastic workloads. We evaluate the proposed algorithm on a multi-task HI benchmark spanning language and vision workloads, demonstrating improved stability and performance under sparse terminal feedback compared to standard importance-weighted methods.

%To further validate the proposed framework under realistic conditions, we construct a large-scale multi-modal benchmark covering heterogeneous tasks and models. The dataset consists of over 70,000 samples spanning language and vision workloads across more than 20 open-source models, capturing practical challenges such as heterogeneous model capability and partial task coverage. This benchmark provides a comprehensive testbed for studying learning stability, system efficiency, and scalability in hierarchical inference.

Our main contributions are as follows.

\begin{itemize}

\item {\it Structured learning formulation.} We formalize multi-layer hierarchical inference as an online learning problem with recursively defined loss and terminal-only feedback, in which the probability of observing a task's loss is policy-dependent and depth-sensitive.

\item {\it Variance-reduced learning algorithm.} We develop a distributed online routing algorithm that integrates Lyapunov optimization for long-term resource constraints with a variance-reduced EXP4-based estimator tailored to the recursive, policy-dependent feedback structure induced by hierarchical routing.
        
\item {\it Theoretical guarantees.} We establish sublinear regret guarantees for the proposed routing algorithm relative to the best fixed policy in hindsight, and ensure that long-term resource constraints are satisfied under stochastic task arrivals.

\item {\it Empirical validation.} We evaluate the proposed method on multi-task hierarchical inference workloads spanning language and vision settings, demonstrating improved learning stability and routing performance under sparse, policy-dependent feedback compared to standard importance-weighted approaches.

\end{itemize}

\section{Related Work}\label{sec:related}

\textbf{Online Learning under Partial Feedback.}
In adversarial and contextual bandit settings, a learner sequentially selects actions and observes only partial feedback, typically the loss of the chosen action. Algorithms such as EXP4 and related expert-based methods~\citep{auer2002nonstochastic,beygelzimer11a,senContextualBanditsStochastic2018,ito} as well as UCB-type approaches~\citep{agarwalb14,linUCB} achieve sublinear regret using importance-weighted loss estimators, where the observed loss is scaled by the inverse of the action-selection probability. While unbiased, these estimators suffer from high variance when action-selection probabilities are small. Beyond standard bandit feedback, feedback models such as partial monitoring~\citep{cesa2006regret} and feedback graphs~\citep{alon2015online} allow limited or indirect observations. However, in these formulations the feedback structure is fixed and exogenous, independent of the learner’s policy. Online learning with long-term constraints has also been studied using Lyapunov and primal–dual methods to ensure feasibility under stochastic dynamics while maintaining sublinear regret~\citep{neely2010stochastic,CayciZE22,liu_efficient_2021}. In contrast, the feedback structure induced by hierarchical routing is both \emph{partial} and \emph{policy-dependent}: the probability of observing loss depends recursively on downstream routing decisions, creating depth-sensitive observability not captured by existing models.

\textbf{Hierarchical Inference (HI).}
Hierarchical inference has been studied in distributed learning and edge computing systems, with most prior work considering two-layer architectures with binary offloading decisions~\citep{yin2018hierarchical}. \cite{al2023case,al2024regret} established regret guarantees for threshold-based policies under single-destination settings, while \cite{beytur2024optimization} integrated contextual bandits with Lyapunov optimization to handle resource constraints. Recent works explore deeper client–edge–cloud hierarchies~\citep{cha2025batching,ma2025hyperion,yuan2025local}, but typically assume static optimization, centralized coordination, or restricted routing topologies. These formulations do not explicitly analyze recursively defined, policy-dependent feedback arising in multi-layer, multi-destination routing under stochastic task arrivals.

% \textbf{LLM Routing.}
% Prior work on LLM routing primarily studied online model selection in centralized systems hosting multiple models on a single server~\citep{wolf2019huggingface,ong2024routellm,chen2023frugalgpt,hu2024routerbench}. These approaches focus on single-hop cost–quality trade-offs and do not consider distributed multi-layer routing in which decisions propagate across layers, nor do they analyze recursively defined, policy-dependent feedback induced by terminal-only supervision and dynamic model placement under memory constraints.

To our knowledge, recursively defined, policy-dependent feedback in hierarchical inference systems has not been formally analyzed. This work develops a variance-reduced learning framework and establishes regret guarantees for this practically motivated setting.

\section{System Model}
\label{sec:sys}
For clarity, we summarize notation in Table \ref{tab:notation-main} in Appendix.

{\bf Online Learning Perspective.} We model HI routing as a decentralized online learning problem with partial and policy-dependent feedback. At each intermediate computing node, arriving jobs generate contextual observations (e.g., task type and confidence score), upon which the node selects a routing action, such as local termination or upstream offloading. 
The objective is to minimize the long-term expected inference error across jobs. The inference loss is defined recursively along the routing path, since downstream routing decisions determine the node at which a job ultimately terminates and thus the final incurred error. Routing decisions also incur upstream resource costs, which are regulated through long-term node-level constraints. Feedback is revealed only if a job is routed to the oracle layer; as a result, the probability of observing loss at a node depends recursively on downstream routing decisions under the global policy. This yields an online learning problem with recursively defined loss and terminal-only, policy-dependent feedback.

{\bf Network Model.} As illustrated in Fig.~\ref{fig:overview}, we consider a set of computing nodes $\mathcal{N}$ partitioned into $K$ non-empty hierarchical layers, denoted by $\mathcal{N}_1, \dots, \mathcal{N}_K$. 
Layer $k=1$ consists of edge devices (clients) that generate inference jobs, layers $1 < k < K$ contain intermediate computing nodes, and layer $K$ represents an oracle node (e.g., a cloud service or human judge) that provides ground-truth supervision. The layers form an inference hierarchy in which a job arriving at a node in layer $k<K$ is executed locally to produce an output and confidence score. Based on this local evaluation, the node either terminates the inference process by accepting its output or offloads the job to a node in layer $k+1$ for further processing. The oracle layer $K$ always terminates the routing process.

\textbf{Job Arrivals.}
The system operates in a slotted time structure where each slot $t$ has a duration $\tau$. 
In each time slot $t$,
the system receives a set of jobs $\mathcal{J}(t)$.
Each job $j \in \mathcal{J}(t)$ is characterized by a task type $y(j)\in\mathcal{Y}$ (e.g., text summarization, image editing) and is initially assigned to a specific entry node $n_1^j\in\mathcal{N}_1$ in the first layer. 
We let $t^j$ denote the arrival slot of job $j$, i.e., $j\in\mathcal{J}(t^j)$.
For any node $n \in \mathcal{N}$, let $\mathcal{I}_n(t)$ denote the distribution of jobs arriving at node $n$ during slot $t$.
We make the following assumption to facilitate the theoretical analysis in Section \ref{sec:methodology}.

\begin{assumption}[Stationary Arrivals]\label{assumption:arrival}
% The sequence of random job sets $\{\mathcal{J}(t)\}_{t>1}$ is i.i.d. over time with bounded first and second moments. Specifically, for all time slots $t$, there exist finite constants $\kappa_1, \kappa_2 > 0$ such that $\mathbb{E}[|\mathcal{J}(t)|] \le \kappa_1$ and $\mathbb{E}[|\mathcal{J}(t)|^2] \le \kappa_2$. 
The sequence $\{\mathcal{J}(t)\}_{t>1}$ is i.i.d. over time. The total arrival count $|\mathcal{J}(t)|$ has bounded moments, i.e., $\mathbb{E}[|\mathcal{J}(t)|] \le \kappa_1$ and $\mathbb{E}[|\mathcal{J}(t)|^2] \le \kappa_2$ for $\kappa_1, \kappa_2 > 0$. 
For each entry node $n \in \mathcal{N}_1$, the arrival distribution is stationary, i.e., $\mathcal{I}_n(t)=\mathcal{I}_n$ for all $t$, and the per-slot arrival count is i.i.d. over time.
\end{assumption}

\textbf{Hierarchical Inference Process.} 
Consider a job $j \in \mathcal{J}(t)$ originating at entry node $n_1^j \in \mathcal{N}_1$. 
The inference process proceeds hierarchically across layers. 
We define $\pi_{n}(\cdot)$ as the online routing policy deployed at node $n$, mapping an arriving job to a probability distribution over its routing decisions. As the policy updates across time slots, the exact distribution applied to job $j$ is bound to its arrival slot $t^j$.
Let $\boldsymbol{\pi}(\cdot)\! =\! \{\pi_n(\cdot)\!: \!n \in \mathcal{N}\}$ denote the joint routing policy. Evaluated at $t^j$, this ensemble of local mappings acts as a joint function of $j$ to dynamically induce a random routing path. For brevity, we hereafter omit the functional arguments, denoting the policies as $\pi_n$ and $\boldsymbol{\pi}$. Let $n_k^{\boldsymbol{\pi},j} \in \mathcal{N}_k$ denote the node in layer $k > 1$ visited by job $j$ (if reached) under $\boldsymbol{\pi}$.
When job $j$ arrives at the node $n=n_k^{\boldsymbol{\pi},j}$, the node first performs local inference. This process yields a confidence score $z_n(j)$, which represents a realization of the random confidence variable $Z_n(j)$. 
Based on $z_n(j)$, the node decides whether to terminate the inference process locally or offload the job to a node in the next layer.
At any node $n\in\mathcal{N}_k$, the offloading action is denoted by $a\in \mathcal{A}_k$, where the action space is defined as $\mathcal{A}_k\triangleq \{0\}\cup \mathcal{N}_{k+1}$. The action $a=0$ denotes local termination, and any $a=n' \in \mathcal{N}_{k+1}$ denotes offloading the job to the upstream node $n'$.
% For a node $n\in\mathcal{N}_k$, action space consists of either local finalization (i.e., no offloading) or offloading the job to a node in the next layer $k+1$. Formally, the action set of a node in layer $k$ is defined as $\mathcal{A}_k\triangleq \{0\}\cup \mathcal{N}_{k+1}$.

For a job $j$ that reaches layer $k < K$, at node $n = n_k^{\boldsymbol{\pi},j} \in \mathcal{N}_k$, we define $\mathbf{O}_n^{\pi_n,j}$ as the random one-hot action vector induced by the policy $\pi_n$.
Its specific realization is represented by the one-hot vector $\mathbf{o}_n^{\pi_n,j} \in \{0,1\}^{|\mathcal{A}_k|}$, defined as
\(
\mathbf{o}_n^{\pi_n,j} = \{o_n^{\pi_n,j}(a)\}_{a \in \mathcal{A}_k}.
\) Note that
$o_n^{\pi_n,j}(0) = 1$ indicates local termination, where node $n$ accepts the local inference result and concludes the job's routing.
For any node $n'\in\mathcal{N}_{k+1}$, $o_{n}^{\pi_n,j}(n')=1$ indicates that the job is offloaded to node $n'$ (i.e., $n_{k+1}^{\boldsymbol{\pi},j}=n'$) for further processing.

This sequence of decentralized routing decisions induces an inference path.
Let $\Omega^{\boldsymbol{\pi},j}$ be the random inference path induced by the joint policy $\boldsymbol{\pi}$, representing the sequence of nodes visited by job $j$. For a given realization of all routing actions, the deterministic path is denoted by $\omega^{\boldsymbol{\pi},j}$. 
We define the exit layer, denoted by $l(\boldsymbol{\pi},j)$, as the layer at which the routing terminates: 
\(
l(\boldsymbol{\pi},j) \triangleq \min\left\{k :\exists\, n\in\mathcal{N}_k \text{ s.t. } o_{n}^{\pi_n,j}(0)=1\right\}
\).
Therefore, the deterministic path is expressed as
\(
\omega^{\boldsymbol{\pi},j}=(n_1^j,n_2^{\boldsymbol{\pi},j},\dots,n_{l(\boldsymbol{\pi},j)}^{\boldsymbol{\pi},j})\).
For any intermediate layer $1\leq k<l(\boldsymbol{\pi},j)$, the subsequent node $n_{k+1}^{\boldsymbol{\pi},j}$ is strictly dictated by the realized offloading action of node $n_k^{\boldsymbol{\pi},j}$. The path concludes at layer $l(\boldsymbol{\pi},j)\leq K$ upon the selection of the local termination action. We assume that routing and execution for jobs arriving in slot $t$ complete within the same slot.

\textbf{Dynamic Model Onloading.}
We generalize the system by allowing periodic model placement updates to accommodate workload shifts.
Every $D>1$ time slots (i.e., $t \pmod D = 1$), each node $n$ can update its set of loaded models, denoted by $\mathcal{M}_{n}(t) \subseteq \mathcal{M}$, where $\mathcal{M}$ is the global set of available models.
The model placement decision is parameterized by a binary vector $\mathbf{x}_{n}(t)=\{x_{n}^m(t)\}_{m\in\mathcal{M}}$, where $x_{n}^m(t)=1$ if model $m$ resides in node $n$'s memory during slot $t$, strictly bounded by local capacity constraints.

{\bf Error and Cost Model.} Consider a job $j \in \mathcal{J}(t)$ of task type $y(j)\in\mathcal{Y}$. Inference error arises from stochastic decoding or evaluation mechanisms~\citep{holtzman2019curious}. We capture this inherent randomness by a noise variable $\zeta$.
For each model $m \in \mathcal{M}$, let $\epsilon(j,m,\zeta)\in\{0,1\}$ be a Bernoulli random variable representing the error incurred when model $m$ processes job $j$. When job $j$ is processed locally at node $n$, the node selects a single model from its loaded set $\mathcal{M}_n(t^j) \triangleq \{m\in\mathcal{M}: x_n^m(t^j)=1\}$ according to a fixed model-selection rule $\sigma_n(\cdot)$. 
% Let $m_n^j $ denote the selected model. 
Let $\sigma_n\!\big(j,\mathbf{x}_n(t^j)\big)\in \mathcal{M}_n(t^j)$ denote the selected model.
The random inference error incurred at node $n$ is then
\begin{align}
B\big(j,\mathbf{x}_n(t^j)\big)=\epsilon\left(j,\sigma_n\!\left(j,\mathbf{x}_n(t^j)\right),\zeta\right),\label{eq:inference_error}
\end{align}
with realization $b\big(j,\mathbf{x}_n(t^j)\big)$. For unsupported modalities, we define $\epsilon(j,m,\zeta)=1$ almost surely. If node $n$ has no loaded model capable of processing task type $y(j)$, we define $B\big(j,\mathbf{x}_n(t^j)\big)=1$. The terminal cloud node $n\in\mathcal{N}_K$ operates as a perfect oracle; thus,
$B\big(j,\mathbf{x}_n(t^j)\big)=0 \text{ for } n\in\mathcal{N}_K$. All expectations are taken with respect to job arrivals, routing randomization, and model stochasticity.

\begin{remark}[Partial Feedback] 
% Each node along the realized routing path executes the job and produces an output. However, the supervisory signal required to determine the inference error is revealed only if the job reaches the cloud layer $K$. 
% Consequently, for any $n \in \mathcal{N}\setminus\mathcal{N}_K$, the realized loss $b(j,\mathbf{x}_n(t^j))$ is observed if and only if $j$ is routed to the oracle. 
% This terminal-only supervision induces partial and policy-dependent learning signals, since the probability of observing feedback depends recursively on downstream routing decisions.
For any $n \in \mathcal{N}\setminus\mathcal{N}_K$, the loss $b(j,\mathbf{x}_n(t^j))$ is observed exclusively if job $j$ reaches layer $K$. This terminal-only supervision induces a partial, policy-dependent learning signal recursively coupled to subsequent routing decisions.
\end{remark}

Upon termination at exit layer $l(\boldsymbol{\pi},j) \leq K$, the system incurs the final inference error $b(j,\mathbf{x}_{n_{l(\boldsymbol{\pi},j)}^{\boldsymbol{\pi},j}}(t^j))$.
Conversely, routing a job to an upstream layer trades potential error reduction for additional resource consumption. 
Specifically, an offloading step from node $n' \in \mathcal{N}_{k-1}$ to node $n \in \mathcal{N}_k$ (where $k>1$) incurs a deterministic and strictly bounded resource cost denoted by $c^j(n', n)$. 
This term explicitly encapsulates both the network communication overhead required to transfer job $j$ and the subsequent computational load incurred at the receiving node $n$. We formally define the total random resource cost incurred by node $n \in \mathcal{N}_k$ ($k>1$) for receiving and processing offloaded jobs during time slot $t$ as
% \footnote{Cost is attributed to the receiver because it executes the job and bears the primary resource consumption.}
\begin{align}
C_{n}^{\boldsymbol{\pi}}(t)=\sum_{j\in\mathcal{J}(t)}\sum_{n'\in\mathcal{N}_{k-1}} c^j(n', n)\mathbbm{1}_{n=n_{k}^{\boldsymbol{\pi},j}}\mathbbm{1}_{n'=n_{k-1}^{\boldsymbol{\pi},j}} .
\label{eq:costC}
\end{align}
% Equations~\eqref{eq:inference_error} and \eqref{eq:costC} form the localized optimization trade-off at any intermediate node $n$: 
% Upon processing job $j$, the node observes the realized confidence score $z_{n}(j)$, which dictates the routing action. 
% A local termination decision ($o_{n}^{\pi_n,j}(0)=1$) restricts the system penalty to the local inference error, whereas an offloading decision incurs deterministic resource consumption in exchange for potentially lower error via upstream models. 
Based on confidence $z_{n}(j)$, node $n$ faces a strict trade-off: local termination ($o_{n}^{\pi_n,j}(0)=1$) incurs only local inference error, while offloading exchanges deterministic resource consumption for potentially reduced upstream error.
The empirical derivation of $z_{n}(j)$ is detailed in Section~\ref{sec:exp}.

\textbf{Problem Formulation.}
The global objective is to minimize the long-term expected inference error across all jobs, subject to long-term resource constraints and strict memory capacities. 
The optimization problem is formulated as
\begin{align}
\min_{\boldsymbol{\pi},\{\mathbf{x}_{n}\!(t)\}_{n\in\mathcal{N}\!,t>1}}  
\limsup_{T \to \infty} \frac{1}{T} \sum_{t=1}^{T}\mathbb{E}\!\!\left[\sum_{j\in\mathcal{J}(t)}\!\! B\!\left(j,\mathbf{x}_{n_{l(\boldsymbol{\pi},j)}^{\boldsymbol{\pi},j}}(t^j)\right)\right],\nonumber
\end{align}
where $n^{\boldsymbol{\pi},j}_{l(\boldsymbol{\pi},j)}$ denotes the specific node at which job $j$ terminates.
This minimization is subject to the following long-term resource and instantaneous memory constraints:
\begin{align}
& \limsup_{T \to \infty} \frac{1}{T} \sum_{t=1}^{T} \mathbb{E}\left[C^{\boldsymbol{\pi}}_{n}(t)\right] \leq \gamma_{n} \tau, \, \forall n\in\mathcal{N}\setminus \mathcal{N}_{1}, \label{eq:constraint_offload} \\
&\sum_{m \in \mathcal{M}} x_{n}^{m}(t) s_m \leq \mu_{n},\, \forall n\in\mathcal{N}.\label{eq:constraint_onload}
\end{align}
Here, $\gamma_{n}$ defines the maximum allowable long-term average resource consumption rate per time slot of length $\tau$, and $\mu_{n}$ represents the physical memory capacity for node $n$. 
Directly solving this stochastic formulation is intractable due to the unknown arrival dynamics and the partial, policy-dependent feedback. The subsequent section decomposes the problem across two distinct timescales: online distributed learning of the routing policy $\boldsymbol{\pi}$, and periodic, workload-aware updates to the model placement $\mathbf{x}_n(t)$.

\section{Methodology}
\label{sec:methodology}

We decompose the hierarchical inference problem into a variance-reduced EXP4 learning framework for dynamic routing (Sections~\ref{sec:lyapunov}--\ref{sec:theory}) and a greedy scheme for periodic model placement (Section~\ref{sec:onloading}). For routing, we apply Lyapunov optimization to enforce the long-term resource constraints~\citep{neely2010stochastic}, and formulate a variance-reduced EXP4 algorithm to overcome the sparse, depth-amplified biased learning signals in multi-layer HI systems.

\subsection{Lyapunov Resource Optimization}
\label{sec:lyapunov}
We use the Lyapunov optimization framework to convert the long-term constraints into virtual queue stability conditions, inducing an online control penalty into the per-slot objective.

\textbf{Virtual Queues.}
% To incorporate the long-term cost constraint into online decision-making, we introduce virtual queues. 
For each node $n \in \mathcal{N}\setminus \mathcal{N}_{1}$, we define a virtual queue $Q_{n}(t)$ that tracks the deviation between instantaneous resource consumption and the per-slot budget $\gamma_{n}\tau$. 
We initialize the stochastic process such that $Q_{n}(1)=0$ with probability $1$.
The virtual queue process evolves as
\begin{equation}
Q_{n}(t+1) = \max\left[Q_{n}(t) + C^{\boldsymbol{\pi}}_{n}(t) - \gamma_{n} \tau, \, 0\right].\nonumber
\end{equation}
Let $\mathbf{Q}(t)=\{Q_n(t)\}_{n \in \mathcal{N}\setminus\mathcal{N}_1}$ be the random vector of virtual queues.
When the realized routing cost exceeds the budget $\gamma_{n}\tau$, the queue grows; otherwise, it decreases. 
\begin{proposition}\label{prop:queue}
For each node $n\in\mathcal{N}\setminus\mathcal{N}_1$, If the virtual queue $Q_{n}(t)$ is mean-rate stable, i.e., 
\(
\lim_{T \to \infty} \frac{\mathbb{E}[Q_{n}(T)]}{T} = 0  ,  
\)
then the resource constraint in Eq.~\eqref{eq:constraint_offload} is satisfied.
\end{proposition}
The proof is provided in Appendix~\ref{app:proposition}. 
By Proposition~\ref{prop:queue}, enforcing Eq.~\eqref{eq:constraint_offload} reduces to ensuring queue stability. To formalize this, we define the Lyapunov function $\mathcal{L}(t) = \frac{1}{2} \sum_{n\in\mathcal{N}\setminus \mathcal{N}_1} (Q_{n}(t))^2$.
Let $\mathbf{q}(t) = \{q_n(t)\}_{n \in \mathcal{N}\setminus\mathcal{N}_1}$ be the realized queue state observed at the beginning of slot $t$. The conditional Lyapunov drift, defined as $\Delta(\mathbf{q}(t)) = \mathbb{E}[\mathcal{L}(t+1)-\mathcal{L}(t) \mid \mathbf{Q}(t) = \mathbf{q}(t)]$, is strictly bounded by:
\begin{align}
\Delta(\mathbf{q}(t)) \le \alpha + \sum_{n\in\mathcal{N}\setminus \mathcal{N}_{1}} q_{n}(t)\mathbb{E}\left[ C^{\boldsymbol{\pi}}_{n}(t)-\gamma_n\tau \right], \label{eq:driftbound}
\end{align}
where $\alpha>0$ is a finite constant determined by the bounded per-slot cost and Assumption~\ref{assumption:arrival} (detailed in Appendix~\ref{app:driftbound}).
% By Proposition~\ref{prop:queue}, minimizing a per-slot objective that incorporates the bound \eqref{eq:driftbound} enforces the constraint in Eq.~\eqref{eq:constraint_offload}.

\textbf{Drift-Plus-Penalty Minimization.}
Following the Lyapunov framework, we jointly stabilize the queues and minimize the expected inference error by optimizing the drift-plus-penalty term: $\Delta(\mathbf{q}(t))+v\mathbb{E}[\sum_{j\in\mathcal{J}(t)} B(j,\mathbf{x}_{n_{l(\boldsymbol{\pi},j)}^{\boldsymbol{\pi},j}}(t^j))]$, where $v>0$ balances performance optimality against constraint satisfaction. 
Applying the bound (Eq.~\eqref{eq:driftbound}) yields the per-slot objective, which seeks the joint policy $\boldsymbol{\pi}$ and model placements $\{\mathbf{x}_{n}(t)\}_{n\in\mathcal{N}}$ that minimize the drift-plus-penalty given the realized queue states at slot $t$:
\begin{align}
\min_{\boldsymbol{\pi},\{\mathbf{x}_{n}(t)\}_{n\in\mathcal{N}}} &\sum_{n \in \mathcal{N}\setminus\mathcal{N}_1} q_{n}(t) \mathbb{E}\left[ C^{\boldsymbol{\pi}}_{n}(t) \right] \nonumber\\
&+ v \mathbb{E}\left[ \sum_{j \in \mathcal{J}(t)} B\left(j,\mathbf{x}_{n_{l(\boldsymbol{\pi},j)}^{\boldsymbol{\pi},j}}(t^j)\right) \right],\label{eq:obj_perslot}
\end{align}

This isolates the long-term stochastic optimization into a sequence of unconstrained decisions in each slot. 
The inference error $B(j,\mathbf{x}_{n_{l(\boldsymbol{\pi},j)}^{\boldsymbol{\pi},j}}(t^j))$ is only observed under terminal-only, partial feedback, therefore, Eq.~\eqref{eq:obj_perslot} cannot be evaluated directly, necessitating an online learning method.

\subsection{Hierarchical Routing Bandits}\label{sec:variance}
We model the decentralized offloading decision at each node as a contextual bandit problem governed by the per-slot objective~\eqref{eq:obj_perslot}. Each node $n$ maintains a routing policy learned online via EXP4~\citep{auer2002nonstochastic}.

% At slot $t$, the policy is realized as a probabilistic action vector $\mathbf{p}_{n_k}(t)\in\mathbb{R}^{|\mathcal{A}_k|}$.

Let $n\in\mathcal{N}\setminus \mathcal{N}_K$ be an intermediate node, and let $\mathcal{U}_n$ be its set of valid routing destinations (i.e., if $n \in \mathcal{N}_k$, then $\mathcal{U}_n = \mathcal{N}_{k+1}$).
We define a set of experts $\mathcal{H}_y$ for each task type $y$, where each expert $h \in \mathcal{H}_y$ applies an offloading threshold $\theta_{h}\in[0,1]$. 
To couple the local thresholding logic with the multi-destination routing, we define the joint expert space $\mathcal{E}_n(y) = \mathcal{H}_y \times \mathcal{U}_n$, with cardinality $|\mathcal{E}_n(y)| = |\mathcal{H}_y| |\mathcal{U}_n|$.
A joint expert $e = (h, n') \in \mathcal{E}_n(y)$ represents a deterministic policy $\pi_{(h,n')}$: for an arriving job $j$ of type $y$ with realized confidence $z_n(j)$, if $z_n(j) < \theta_h$, recommend offloading to $n'$; otherwise, recommend local termination.

The node maintains a probability weight $w_{n,y}^{e}(t)$ for each joint expert, satisfying $\sum_{e \in \mathcal{E}_n(y)} w_{n,y}^{e}(t) = 1$. For a job $j$ arriving at slot $t^j$, the probability of offloading to upstream node $n' \in \mathcal{U}_n$ aggregates the weights of all experts advocating for $n'$ whose thresholds exceed the local confidence:
\begin{align}
p^j_{n}(n') &= \sum_{h \in \mathcal{H}_{y(j)}}  w_{n,y(j)}^{(h,n')}(t^j) \mathbbm{1}_{\theta_h > z_n(j)}.
\end{align}
The probability of local termination (action $a=0$) is
\begin{align}
p^j_{n}(0) &= \sum_{e=(h,n') \in \mathcal{E}_n(y(j))}  w_{n,y(j)}^{e}(t^j) \mathbbm{1}_{\theta_h \leq z_n(j)}.
\end{align}
% Recall $\mathbf{O}_{n}^{\pi_n,j} \in \{0, 1\}^{|\mathcal{U}_n| + 1}$ is the random one-hot action vector generated by the node $n$ for job $j$.
% Its expectation satisfies $\mathbb{E}[O_{n}^{\pi_n,j}(a)] = p^j_{n}(a)$ for all actions $a \in  \{0\}\cup\mathcal{U}_n $.
To guarantee sufficient exploration, the final action distribution is mixed with a uniform strategy: $(1-\lambda)p_{n}^j(a) + \lambda / (|\mathcal{U}_n| + 1)$ for $\lambda\in(0,1)$.
The expert weights are updated via an exponential weighting scheme based on the cumulative loss $\hat{g}_{n,y}^{(h,n')}(t-1)$ observed up to slot $t-1$:
\begin{align}
w_{n,y}^{(h,n')}(t)
=
\frac{\exp\!\big(-\eta \hat{g}_{n,y}^{(h,n')}(t-1)\big)}
{\sum_{e' \in \mathcal{E}_n(y)}
\exp\!\big(-\eta \hat{g}_{n,y}^{e'}(t-1)\big)}.\nonumber
\end{align}
The cumulative loss for task type $y$ is defined over all historical jobs of that specific type:
\begin{align}
\hat{g}_{n,y}^{(h,n')}(t-1)=\sum_{t'=1}^{t-1} \sum_{j\in\mathcal{J}_n(t')} \hat{f}_{n}^{n'}(j,\pi_{h})\mathbbm{1}_{n\in \omega^{\boldsymbol{\pi},j}}
\mathbbm{1}_{y(j)=y},\nonumber
\end{align}
where $\omega^{\boldsymbol{\pi},j}$ is the realized routing path, and $\hat{f}_{n}^{n'}(j,\pi_{h})$ is the unbiased estimate of the true per-job loss $f_{n}^{n'}(j,\pi_{h})$ incurred under the threshold policy $\pi_h$. Specifically, $\pi_h$ dictates routing if the confidence satisfies $\theta_h > z_n(j)$, and mandates local termination otherwise.

\subsection{Variance-Reduced Loss Estimation}
To build the loss estimator, we first define the true loss corresponding to the objective~\eqref{eq:obj_perslot} assuming full observability. Consider a job $j$ arriving at node $n$ with a realized confidence score $z_n(j)$. For an expert policy $\pi_h$, the full-feedback loss of routing job $j$ to upstream node $n'$ is defined as
\begin{align}
f_{n}^{n'}(j,\pi_h) &= v\mathbbm{1}_{\theta_h \leq z_n(j)}b(j,\mathbf{x}_{n}(t^j)) \\
&\quad+ \mathbbm{1}_{\theta_h > z_n(j)}\Big( q_{n'}(t^j)c^j(n,n') + \bar{f}_{n'}(j,\pi_{n'})\Big), \nonumber
\end{align}
where $\bar{f}_{n'}(j,\pi_{n'})$ denotes the expected loss incurred at the upstream node $n'\in \mathcal{U}_n$, with the expectation taken over all downstream routing randomness induced by $\pi_{n'}$, conditioned on the specific job $j$ and the realized queue states.
To formalize this expected loss, we define the random per-job loss incurred by node $n$ under policy $\pi_n$ as
\begin{align}
F_n(j, \pi_n)& = v O_n^{\pi_n,j}(0) b(j, \mathbf{x}_n(t^j))\\
&+ \sum_{a \in \mathcal{U}_n} O_n^{\pi_n,j}(a) \left( q_a(t^j)c^j(n,a) + \bar{f}_a(j, \pi_a) \right).\nonumber
\end{align}
Taking the exact expectation with respect to the random action vector $\mathbf{O}_n^{\pi_n,j}$ yields the deterministic expected loss. 
Starting from the cloud layer (where the terminal loss is $\bar{f}_{n}(j,\pi_{n})=0,n\in\mathcal{N}_K$), any intermediate node $n$ computes its expected loss recursively over its deterministic action distribution $\{p_n^j(a)\}_{a\in\mathcal{U}_n\cup \{0\}}$:
\begin{align}
\bar{f}_{n}&(j,\pi_{n})=\mathbb{E}[F_n(j, \pi_n)] =vp_{n}^j(0)b(j,\mathbf{x}_{n}(t^j))\\
&+\sum_{a\in\mathcal{U}_{n}}  p_{n}^j(a) \left(q_{a}(t^j)c^j(n,a)+\bar{f}_{a}(j,\pi_a)\right).\nonumber
\end{align}
This recursive structure mirrors the per-slot objective: local termination incurs the local inference error; routing yields both a resource cost and a subsequent processing cost.

Let $\mathbbm{1}_{\text{fb},n}(j)=\mathbbm{1}\{n\in\Omega^{\boldsymbol{\pi},j} \text{ and } l(\boldsymbol{\pi},j)=K\}$ denote the node-conditioned feedback indicator.
A naive unbiased estimator of the expert loss $f_{n}^{n'}(j,\pi_h)$ is constructed via importance sampling as the random variable $\hat{F}_{\text{naive},n}^{n'}(j,\pi_h)$:
\begin{align}
\hat{F}_{\text{naive},n}^{n'}(j,\pi_h)=\mathbbm{1}_{\text{fb},n}(j)\frac{f_{n}^{n'}(j,\pi_h)}{\rho_{n}^{\boldsymbol{\pi}}(j)},\label{eq:naive}
\end{align}
Its specific realization, $\hat{f}_{\text{naive},n}^{n'}(j,\pi_h)$, serves as a valid loss term for computing $\hat{g}_{n,y}^{(h,n')}(t-1)$.
Here, $\rho_{n}^{\boldsymbol{\pi}}(j)$ is the marginal probability that job $j$ reaches the oracle from node $n$ under the global policy $\boldsymbol{\pi}$. 
This probability is computed recursively backward: $\rho_{n}^{\boldsymbol{\pi}}(j) = 1$ for the cloud node $n \in \mathcal{N}_K$, and $\rho_{n}^{\boldsymbol{\pi}}(j) = \sum_{a \in \mathcal{U}_n} p_n^j(a) \rho_{a}^{\boldsymbol{\pi}}(j)$ for any intermediate node $n\in\mathcal{N}\setminus\mathcal{N}_{K}$.
We notice that $\rho_{n}^{\boldsymbol{\pi}}(j)$ decays multiplicatively with routing depth, therefore, the naive estimator \eqref{eq:naive} suffers from severe depth-amplified variance.

% \textbf{Variance-Reduced Loss.}
% The naive estimator \eqref{eq:naive} can suffer from high variance when $\rho_{n}^{\boldsymbol{\pi}}(j)$ is small.
To reduce the variance of the loss estimator and stabilize the learning process, we propose a \textbf{task-conditioned, variance-reduced expert loss estimator}, defined as
\begin{align}
\hat{F}_{\text{vr}, n}^{n'}(j,\pi_h)&=\mathbbm{1}_{\text{fb},n}(j)\frac{f_{n}^{n'}(j,\pi_h)-\bar{f}_{n,y(j)}^{n'}(\pi_h)}{\rho^{\boldsymbol{\pi}}_{n}(j)}\nonumber\\
&+\bar{f}_{n,y(j)}^{n'}(\pi_h).\label{eq:vr_loss}
\end{align}
We denote its realized value by $\hat{f}_{\text{vr},n}^{n'}(j,\pi_h)$.
Since inference errors and routing costs vary significantly across heterogeneous workloads, the baseline must be task-conditioned. 
The baseline $\bar{f}_{n, y}^{n'}(\pi_h)$ estimates the conditional theoretical loss $\mathbb{E}_{J\sim \mathcal{I}_{n}(t) | y(J)=y}[f_{n}^{n'}(J,\pi_h)]$, computed from historical observations matching task type $y$ (detailed in Appendix \ref{app:baseline_compute}).
With Eq. \eqref{eq:vr_loss}, expert weights are updated using the baseline $\bar{f}_{n,y}^{n'}(\pi_h)$ even in the absence of feedback, while the residual term corrects for unbiasedness when feedback is observed. 
This preserves unbiasedness while significantly reducing estimator variance. We refer to this integrated algorithm as \textbf{VR-Ly-EXP4} (pseudocode provided in Appendix~\ref{app:code-VR}).

% \begin{assumption}\label{assumption:est}
% We assume the cost predictor is non-negative and does not drastically overestimate the true cost, specifically $0 < \hat{X}_t < 2X_j$.
% \end{assumption}
\begin{lemma}[Reduced Variance]\label{lemma:variance}
Let $n\in\mathcal{N}_k$ with $1\leq k<K$ and $n'\in\mathcal{N}_{k+1}$.
For a specific job $j$ with task type $y(j)=y$, if the baseline estimator $\bar{f}_{n,y}^{n'}(\pi_h)$ is strictly positive and does not overestimate twice the true expert loss of that job, i.e.,
\(
0 < \bar{f}_{n,y}^{n'}(\pi_h) \leq 2f_{n}^{n'}(j,\pi_h)
\),\footnote{This condition ensures the residual term remains smaller than the naive loss to guarantee variance reduction.}
then the variance-reduced estimator has no larger variance than the naive estimator:
% with respect to the random feedback indicator $\mathbbm{1}_{\text{fb},n}(j)$:
\begin{align}
Var\left(\hat{F}_{\text{vr}, n}^{n'}(j,\pi_h)\right)\leq Var\left(\hat{F}_{\text{naive}, n}^{n'}(j,\pi_h)\right).
\end{align}
\end{lemma}
The proof is provided in Appendix \ref{app:variance}. 
We now quantify the impact of this variance reduction on learning performance.

\subsection{Theoretical Guarantees}\label{sec:theory}
We evaluate the regret over a horizon of $\Gamma=\sum_{t=1}^T |\mathcal{J}(t)|$ jobs.
Because the expert space is conditioned on the task type, we evaluate the regret for each specific task type $y \in \mathcal{Y}$. 
We define the expected regret of node $n \in \mathcal{N}\setminus\mathcal{N}_K$ for task type $y$, relative to the best fixed expert in hindsight as
\begin{align}
R_{n, y}(\Gamma) &= \sum_{j=1}^\Gamma\mathbb{E} \left[  \mathbbm{1}_{y(j)=y} \mathbbm{1}_{n \in \Omega^{\boldsymbol{\pi},j}} F_n(j,\pi_n) \right] \nonumber \\
&- \min_{(h, n') \in \mathcal{E}_{n}(y)} \sum_{j =1}^\Gamma\mathbb{E} \left[  \mathbbm{1}_{y(j)=y} \mathbbm{1}_{n \in \Omega^{\boldsymbol{\pi}, j}} f_{n}^{n'}(j, \pi_h) \right],\nonumber
\end{align}
where the expectation is taken over the stochastic job arrivals, routing decisions, and partial feedback indicators.
Our analysis makes no stochastic or stationarity assumptions on the induced losses. The guarantees rely on the conditional unbiasedness of the estimator $\hat{F}_{\text{vr}, n}^{n'}(j,\pi_h)$, and its boundedness, which is jointly guaranteed by Assumption \ref{assumption:arrival} and $\lambda$-exploration ($\rho_n^\pi(j) > 0$).

\begin{theorem}[Regret Bound]\label{theorem:regret}
For $\eta > 0$, the expected node-level regret $R_{n,y}(\Gamma)$ relative to the best fixed threshold-destination pair $(h^*, n'^{*}) \in \mathcal{E}_n(y)$ is strictly bounded by
\begin{align}
R_{n,y}(\Gamma) \le \frac{\ln(|\mathcal{E}_n(y)|)}{\eta} + \eta \sum_{j=1}^\Gamma \mathbb{E}\Big[ \mathbbm{1}_{y(j)=y} \mathbbm{1}_{n \in \Omega^{\boldsymbol{\pi},j}} V_{n,y}(j) \Big],\nonumber
\end{align}
where $V_{n,y}(j)$ is a variance proxy term defined as
\(
V_{n,y}(j) = \sum_{(h,n') \in \mathcal{E}_n(y)} w_{n,y}^{h,n'}(t^j) \left(\hat{F}_{\text{vr}, n}^{n'}(j, \pi_h)\right)^2  + \max_{(h,n') \in \mathcal{E}_n(y)} \left(\hat{F}_{\text{vr}, n}^{n'}(j, \pi_{h})\right)^2.
\)
\end{theorem}
Because the proxy term $V_{n,y}(j)$ scales directly with $(\hat{F}_{\text{vr}, n}^{n'}(j, \pi_h))^2$, mitigating estimator variance fundamentally tightens the regret bound and improves learning stability.
By setting $\eta \propto 1/\sqrt{\Gamma}$, the regret is bounded by $\mathcal{O}(\sqrt{\Gamma})$, which ensures mean-rate stability of the virtual queues.

\textbf{System-level optimality.}
We define the average job arrival rate as $\bar{A} = \Gamma/T$.
For a given sequence of deterministic distributed model placements $\mathbf{x}=\{\mathbf{x}_{n}(t)\}_{n\in\mathcal{N},t=1,\dots,T}$, we define the time-averaged expected system inference error under global policy $\boldsymbol{\pi}$ as $\Phi_\Gamma(\boldsymbol{\pi}) = \frac{1}{\Gamma}\sum_{j=1}^\Gamma \mathbb{E}[B(j, \mathbf{x}_{n_{l(\boldsymbol{\pi},j)}^{\boldsymbol{\pi}, j}}(t^j))]$. Let $\Psi_{max}$ denote a uniform upper bound on the expected variance term from Theorem \ref{theorem:regret} across all task types $y \in \mathcal{Y}$ and intermediate nodes $n \in \mathcal{N} \setminus \mathcal{N}_K$ (details in Appendix \ref{app:corollary}).

\begin{corollary}[Near-Optimality of VR-Ly-EXP4]\label{cor:optimality}
Let $\boldsymbol{\pi}^*$ be the true optimal continuous routing policy, and let $\varepsilon_0$ denote the discretization error of the expert grid such that $\Phi_\Gamma(\boldsymbol{\pi}_{best}) - \Phi_\Gamma(\boldsymbol{\pi}^*) \le \varepsilon_0$, where $\boldsymbol{\pi}_{best}$ is the best fixed combination of discretized experts.
By setting the learning rate to $\eta = \sqrt{\frac{|\mathcal{Y}| \ln |\mathcal{E}_{max}|}{\Gamma \Psi_{max}}}$, the optimality gap of the proposed VR-Ly-EXP4 algorithm is strictly bounded by:
$$\Phi_\Gamma(\boldsymbol{\pi}) - \Phi_\Gamma(\boldsymbol{\pi}^*) \le \frac{\alpha}{v \bar{A}} + \frac{2 |\mathcal{N}|}{v} \sqrt{\frac{|\mathcal{Y}| \Psi_{max} \ln|\mathcal{E}_{max}|}{\Gamma}} + \varepsilon_0,$$
where $|\mathcal{E}_{max}|$ is the maximum cardinality of the joint expert space across all nodes and task types.
\end{corollary}
The bound in Corollary~\ref{cor:optimality} decomposes the optimality gap into three components: (1) the fundamental Lyapunov drift-plus-penalty trade-off $\frac{\alpha}{v \bar{A}}$; (2) the time-averaged learning regret ($\Psi_{max}$) of the distributed bandits, which vanishes at a rate of $\mathcal{O}(1/\sqrt{\Gamma})$; and (3) the irreducible threshold discretization penalty $\varepsilon_0$. Crucially, the control parameter $v$ dictates the system's operational trade-off: while taking $v, \Gamma \to \infty$ drives the performance gap to the discretization floor $\varepsilon_0$, a larger $v$ linearly inflates the virtual queue bounds, thereby exacerbating transient resource deviations.

\subsection{Greedy Model Onloading}\label{sec:onloading}
Every $D$ time slots ($t \pmod D = 1$), at the beginning of slot $t$, node $n$ updates its model placement to maximize expected local execution performance for the current job arrival distribution, subject to its memory capacity $\mu_n$ and accounting for the cost of loading new models relative to the previous placement $\mathcal{M}_n(t-1)$. Let $\mathcal{M}_n(t)\subseteq\mathcal{M}$ denote the set of models loaded at node $n$ during slot $t$. We quantify the quality of a placement at node $n$ by the utility:
\begin{align}
U_n(\mathcal{M}_n(t)) &= \mathbb{E}_{J\sim\mathcal{I}_n(t)} \Big[ 1 - \min_{m \in \mathcal{M}_n(t)} \epsilon(J, m,\zeta) \Big]\nonumber \\
&- \nu \sum_{m \in \mathcal{M}_n(t)} s_m \mathbbm{1}_{m \notin \mathcal{M}_n(t-1)}  ,\label{eq:utility}
\end{align}
where $s_m$ is the memory footprint of model $m$, $\nu\ge 0$ scales the switching penalty, and $\mathbbm{1}_{m\notin \mathcal{M}_n(t-1)}$ isolates newly loaded models. For any subset $\mathcal{S}\subseteq\mathcal{M}$, the marginal gain of adding model $m'$ to node $n$ is the discrete derivative
$\Delta U_n(m’\mid \mathcal{S})= U_n(\mathcal{S}\cup\{m’\})-U_n(\mathcal{S})$, i.e.,
\begin{align}
\Delta U_n(m' \mid &\mathcal{S}) = \mathbb{E}_{J\sim\mathcal{I}_n(t)} \Big[ \min_{m \in \mathcal{S}} \epsilon(J, m,\zeta) \\
&- \min_{m \in \mathcal{S} \cup \{m'\}} \epsilon(J, m,\zeta) \Big] \nonumber
- \nu s_{m'} \mathbbm{1}_{m' \notin \mathcal{M}_n(t-1)}.
\label{eq:marginal_gain_onload2}
\end{align}
Maximizing $U_n(\cdot)$ under the memory constraint yields a \textit{knapsack-constrained submodular maximization problem} \citep{sviridenko2004note} (See Appendix~\ref{app:submodular} for the submodularity proof and the associated approximation guarantee). 
We therefore construct $\mathcal{M}_n(t)$ using a marginal-density greedy rule: initialize $\mathcal{M}_n(t)\leftarrow \emptyset$ and iteratively add
\begin{align}
m^\star \in &\argmax_{m\in\mathcal{M}\setminus\mathcal{M}_{n}(t)}
\frac{\Delta U_n(m\mid \mathcal{M}_{n}(t))}{s_m},
\end{align}
subject to the feasibility condition $\sum_{m \in \mathcal{M}_n(t) \cup \{m^\star\}} s_m \le \mu_n$ and the non-negativity condition $\Delta U(m^\star \mid \mathcal{M}_n(t)) \ge 0$. The greedy accumulation terminates when no feasible candidate remains.
Full pseudocode is provided in Appendix~\ref{app:onload}

% The proof is provided in Appendix \ref{app:submodular}.
% Building on this submodular property, the greedy strategy guarantees a performance bound that strictly depends on the memory footprint of the largest candidate model. 

\section{Experiments}
\label{sec:exp}

\begin{table*}[t]
\centering
\caption{Offloading performance across multi-layer hierarchical systems. Our proposed VR-Ly-EXP4 outperforms all baseline methods across all settings, achieving the lowest inference error and the highest hit rate.}
\label{tab:offload}
\resizebox{0.99\linewidth}{!}{
% Updated column definition: 1 method column + 3 groups of 3 columns (ccc)
\begin{tabular}{@{}l|ccc|ccc|ccc@{}}
\toprule
\multirow{2}{*}{Methods} 
    & \multicolumn{3}{c|}{3-layer (4-2-1)} 
    & \multicolumn{3}{c|}{4-layer (8-4-2-1)} 
    & \multicolumn{3}{c}{5-layer (16-8-4-2-1)} \\ 
\cmidrule(l){2-10} % Span now covers columns 2 through 10
    & Feedback Rate & Hit Rate ($\uparrow$) & Error Rate ($\downarrow$)
    & Feedback Rate & Hit Rate ($\uparrow$) & Error Rate ($\downarrow$) 
    & Feedback Rate & Hit Rate ($\uparrow$) & Error Rate ($\downarrow$)\\ 
\midrule
Random              & 0.0146 & 0.0 & 0.4805$\pm$0.012 & 0.0017 & 0.0 &  0.4793$\pm$0.008 & 0.0002 & 0.0 & 0.4705$\pm$0.011 \\
Round-Robin          &  0.0146   & 0.0          & 0.4627$\pm$0.014          & 0.0019    & 0.0    & 0.4603$\pm$0.002 &  0.0002   &  0.0   & 0.4693$\pm$0.004 \\
Pure Local           &  0.0   & 0.0          & 0.4820$\pm$0.013          &  0.0   &  0.0   & 0.4803$\pm$0.001 &   0.0  &  0.0   & 0.4680$\pm$0.010 \\
Ly-EXP4              & 0.2038    & 0.407          & 0.3433$\pm$0.012          & 0.1720    & 0.399    & 0.3313$\pm$0.003 & 0.1691    &  0.3949   & 0.3222$\pm$0.005 \\ 
\midrule
VR-Ly-EXP4-LocalLoss &  0.2226   & 0.432          & 0.3207$\pm$0.006          & 0.2180    & 0.423    & 0.3117$\pm$0.007 & 0.2212    &  0.4346   & 0.3103$\pm$0.001 \\
\rowcolor{highlight}
VR-Ly-EXP4           & 0.2120    & \textbf{0.442}          & \textbf{0.3172}$\pm$0.012          &  0.2142   & \textbf{0.443}    & \textbf{0.3065}$\pm$0.002 &  0.2089   &  \textbf{0.4456}   & \textbf{0.2923}$\pm$0.001 \\ 
\bottomrule
\end{tabular}}
\vspace{-0.1in}
\end{table*}

% onloading experiments

% \begin{figure}
%     \centering
%     \includegraphics[width=0.7\linewidth]{Amsterdam.jpg}
%     \caption{placeholder, learning progress}
%     \label{fig:placeholder}
% \end{figure}

\textbf{Datasets and Models.}
We evaluate the proposed HI system using a large-scale multi-modal, multi-task, and multi-model benchmark extracted from RouterBench~\citep{hu2024routerbench} and VL-RouterBench~\citep{huang2025vl}. The dataset contains 79{,}988 job samples spanning 114 task types across both language and vision--language settings, where 49{,}448 samples are text-only and 30{,}540 are vision--language.
Each sample contains the results of 23 open-source models with heterogeneous capabilities and scales. 
Among them, 8 models support text-only inputs and 15 support vision--language inputs. 
% Model sizes range from GPT-2-Large (0.7B) \citep{radford2019language} to InternVL2.5-78B \citep{chen2024internvl}.
Additional details on the dataset are provided in the Appendix \ref{app:exp}.
We provide the dataset and experimental code \href{https://anonymous.4open.science/r/llmooo-572B/README.md}{here}.

\textbf{System Setup.}
We evaluate three hierarchical network topologies spanning $K=3$ to $K=5$ layers, structured such that the node count doubles layer by layer. We set a uniform resource consumption budget of $\gamma_n = 0.4$ across all intermediate nodes. Node memory capacities $\mu_n$ scale with the hierarchical depth from 20 to 200 units, where each unit accommodates one billion model parameters. Complete topological specifications are detailed in Appendix~\ref{app:exp}.
The job arrival process at each entry node $n \in \mathcal{N}_1$ follows a non-i.i.d. task distribution drawn from a Dirichlet distribution with concentration parameter $\alpha=1$. 
For a job $j$ processed at node $n$, the confidence score $z_n(j)$ is sampled from a Gaussian distribution centered at the expected error rate of the best-performing model currently loaded at node $n$ for task type $y(j)$. In practical LLM deployments, this score serves as a theoretical proxy for predictive uncertainty measures, such as maximum softmax probability or predictive entropy. 
Each evaluation trace sequentially processes $20{,}000$ jobs sampled from the global dataset. To guarantee statistical significance, all reported metrics are averaged across 5 independent randomized runs.

\textbf{Evaluation of Hierarchical Offloading.}
We evaluate the offloading performance of the proposed routing algorithm against various baselines across hierarchical networks of increasing depth. Table~\ref{tab:offload} summarizes the comparative results for 3-layer, 4-layer, and 5-layer configurations.

\textit{(1) Baselines and Metrics.} We evaluate VR-Ly-EXP4 against three classes of routing schemes.
First, we consider static heuristics:
\emph{Pure Local} (forces all execution at the entry node), \emph{Random}, and \emph{Round-Robin} offloading. To ensure a fair comparison, the routing probabilities for these baselines are calibrated to satisfy the node-level resource constraints (detailed in Appendix~\ref{app:exp}).
The second is Ly-EXP4 \citep{beytur2024optimization}, which integrates Lyapunov optimization with standard EXP4 but lacks any variance reduction mechanism. Finally, we evaluate VR-Ly-EXP4-LocalLoss, an ablation that removes recursive upstream loss ($\bar{f}_n(j,\pi_n)$). 
Performance is evaluated using three metrics: (i) the average \textit{inference error rate}; (ii) the \textit{hit rate}, defined as the fraction of inherently difficult tasks (approximately 11\% of the workload where every model yields an incorrect result) successfully routed to the terminal oracle layer; and (iii) the \textit{feedback rate}, which measures the overall proportion of jobs reaching the oracle to generate learning signals.

\textit{(2) Interpretation of the results.}
As shown in Table~\ref{tab:offload}, VR-Ly-EXP4 outperforms all baseline methods across all hierarchical depths, consistently achieving the lowest inference error and the highest hard-job hit rate. The static heuristics exhibit a complete inability to identify difficult tasks, yielding a $0.0$ hit rate and high inference errors. As the network deepens from 3 to 5 layers, the feedback rate for random and round-robin routing decays (from $0.0146$ to $0.0002$). This depth-amplified sparsity demonstrates why learning is ineffective in multi-layer architectures; while introducing standard bandit learning via Ly-EXP4 improves upon the heuristics, it remains bottlenecked by these sparse signals.

VR-Ly-EXP4 stabilizes the learning process using the variance reduction mechanism. 
The performance gap between Ly-EXP4 and VR-Ly-EXP4 persists as the hierarchy deepens, with VR-Ly-EXP4 maintaining a hit rate above $0.44$ across all settings. VR-Ly-EXP4 also consistently outperforms the \textit{LocalLoss} ablation, validating the necessity of incorporating expected upstream loss. This ensures intermediate nodes can accurately evaluate routing destinations rather than comparing the resource cost solely.
As the architectural depth increases from 3 to 5 layers, the average inference error decreases across all methods. 
This trend is driven by the exponential growth in the total node count (from 7 to 15 to 31 nodes), which accommodates a globally richer and more diverse model placement across the network.
This expansion affords each job a broader set of upstream execution options, increasing the probability of encountering a highly proficient model and leading to improved overall performance for all routing methods.

Table~\ref{tab:offload} evaluates all routing methods under our proposed greedy placement strategy. 
Extensive ablation studies on alternative placement strategies for each method are provided in Appendix~\ref{app:extra_results}.

\section{Conclusion}
This paper formalized the problem of online learning for multi-layer hierarchical inference under partial, policy-dependent feedback. To mitigate the depth-amplified variance in deep routing architectures, we propose VR-Ly-EXP4, a distributed algorithm integrating Lyapunov optimization with a task-conditioned, variance-reduced EXP4 estimator. 
We establish sublinear regret bounds and prove near-optimality under stochastic arrivals. 
Extensive experiments on large-scale workloads demonstrate that VR-Ly-EXP4 dominates other standard baselines, leading to improved learning stability and lower average inference error across all tasks received by the system.

\bibliography{uai2026-template}

\newpage

\onecolumn

\title{Online Learning for Multi-Layer Hierarchical Inference under Partial and Policy-Dependent Feedback \\(Appendix)}
\maketitle

\appendix

\section{Algorithm Pseudo Code}\label{app:code}

\subsection{Pseudo Code of VR-Ly-EXP4}\label{app:code-VR}

\begin{algorithm}[h]
\caption{VR-Ly-EXP4 (Routing) at node $n\in\mathcal{N}_k$, $k<K$}
\label{alg:vr_ly_exp4}
\resizebox{0.77\linewidth}{!}{%
\begin{minipage}{\linewidth}
\begin{algorithmic}[1]
\REQUIRE Task set $\mathcal{Y}$; destination set $U_n=\mathcal{N}_{k+1}$; expert grids $\{H_y\}_{y\in\mathcal{Y}}$ with thresholds $\{\theta_h\}_{h\in H_y}$; control $v$; learning rate $\eta$; exploration $\lambda\in(0,1)$.
\STATE \textbf{Joint experts:} $\mathcal{E}_n(y)\triangleq H_y\times U_n$ for each $y\in\mathcal{Y}$. 
\STATE \textbf{Initialize:} For all $y\in\mathcal{Y}$ and $e=(h,n')\in \mathcal{E}_n(y)$, set 
\(
\hat g_{n,y}^{e}(1)=0, \bar{f}_{n}^{n'}(\pi_h)=0.
\)
\FOR{each time slot $t=1,2,\ldots,T$}
    \STATE Observe current virtual-queue realizations $\{q_{n'}(t)\}_{n'\in U_n}$.
    \FOR{each task type $y\in\mathcal{Y}$}
        \STATE Update EXP4 weights over $(h,n')\in\mathcal{E}_n(y)$:
\(
w_{n,y}^{(h,n')}(t)
=
\frac{\exp\!\big(-\eta \hat{g}_{n,y}^{(h,n')}(t-1)\big)}
{\sum_{e' \in \mathcal{E}_n(y)}
\exp\!\big(-\eta \hat{g}_{n,y}^{e'}(t-1)\big)}.
\)
    \ENDFOR

    \FOR{each arriving job $j\in\mathcal{J}(t)$ with the entry node $n_1^j=n$}
        \STATE Execute local inference for job $j$ and observe context $(y(j), z_n(j))$, where $z_n(j)$ is the realized confidence score.
        \STATE Form action probabilities $p_n^j(\cdot)$ from weights:
        \begin{align*}
p^j_{n}(n') &= \sum_{h \in \mathcal{H}_{y(j)}}  w_{n,y(j)}^{(h,n')}(t^j) \mathbbm{1}_{\theta_h > z_n(j)}, \forall n'\in \mathcal{U}_n, \\
p^j_{n}(0) &= \sum_{e=(h,n') \in \mathcal{E}_n(y(j))}  w_{n,y(j)}^{e}(t^j) \mathbbm{1}_{\theta_h \leq z_n(j)}.
\end{align*}
        \STATE Exploration mix: $\tilde p_n^j(a)=(1-\lambda)p_n^j(a)+\lambda/(|\mathcal{U}_n|+1)$ for all $a\in \mathcal{U}_n\cup\{0\}$. 
        \STATE Sample action $a\sim \{\tilde p_n^j(a)\}_{a\in\mathcal{U}_n\cup\{0\}}$ and execute:
        \IF{$a=0$}
            \STATE Terminate locally (accept local output).
        \ELSE
            \STATE Offload job $j$ to upstream node $a\in \mathcal{U}_n$.
        \ENDIF

        \STATE \textit{Recursive terms.}
        \STATE Obtain $\{\rho_a^{\boldsymbol{\pi}}(j)\}_{a\in \mathcal{U}_n}$ (with $\rho_n^{\boldsymbol{\pi}}(j)=1$ at cloud $n\in\mathcal{N}_K$), set
        \(
        \rho_n^{\boldsymbol{\pi}}(j)= \sum_{a\in \mathcal{U}_n}\tilde p_n^j(a)\,\rho_a^{\boldsymbol{\pi}}(j).
        \)
        \STATE Obtain $\{\bar{f}_{a}(j,\pi_{a})\}_{a\in \mathcal{U}_n}$, set
\(
\bar{f}_{n}(j,\pi_{n})=vp_{n}^j(0)b(j,\mathbf{x}_{n}(t^j))+\sum_{a\in\mathcal{U}_{n}}  p_{n}^j(a) \left(q_{a}(t^j)c^j(n,a)+\bar{f}_{a}(j,\pi_a)\right).
\)
        \FOR{each $e=(h,u)\in E_n(y(j))$}
            \STATE Compute expert loss
\(
\hat{F}_{\text{vr}, n}^{n'}(j,\pi_h)=\mathbbm{1}_{\text{fb},n}(j)\frac{f_{n}^{n'}(j,\pi_h)-\bar{f}_{n,y(j)}^{n'}(\pi_h)}{\rho^{\boldsymbol{\pi}}_{n}(j)}
+\bar{f}_{n,y(j)}^{n'}(\pi_h)
\), and obtain the realization $\hat{f}_{\text{vr}, n}^{n'}(j,\pi_h)$ according to $\mathbbm{1}_{\text{fb},n}(j)$.
            \STATE Accumulate: $\hat{g}_{n,y}^{(h,n')}(t-1)=\sum_{t'=1}^{t-1} \sum_{j\in\mathcal{J}_n(t')}\hat{f}_{\text{vr}, n}^{n'}(j,\pi_h)\mathbbm{1}_{n\in \omega^{\boldsymbol{\pi},j}}
\mathbbm{1}_{y(j)=y}.$
            \IF{$\mathbf{1}_{\mathrm{fb},n}(j)=1$}
                \STATE Baseline update:
                \(
                \bar f_{n,y(j)}^{(h,n')}= (1-\eta_b)\bar f_{n,y(j)}^{(h,n')}+\eta_b\,\frac{f_n^{n'}(j,\pi_h)}{\rho_n^\pi(j)}.
                \)
            \ENDIF
        \ENDFOR
    \ENDFOR
\ENDFOR
\end{algorithmic}
\end{minipage}%
}
\end{algorithm}

\subsection{Pseudo Code of Greedy Model Placement}\label{app:onload}
\begin{algorithm}[h]
\caption{Greedy Model Onloading at node $n$ (every $D$ slots)}
\label{alg:greedy_onloading}
\begin{algorithmic}[1]
\REQUIRE $\mathcal{M}$, sizes $\{s_m\}_{m\in\mathcal{M}}$, budget $\mu_n$, penalty $\nu$, previous set $\mathcal{M}_n(t-1)$
\STATE Define utility
\[
U_n(\mathcal{S}) = \mathbb{E}_{J\sim\mathcal{I}_n(t)} \Big[ 1 - \min_{m \in \mathcal{S}} \epsilon(J, m,\zeta) \Big]
- \nu \sum_{m \in \mathcal{S}} s_m \mathbbm{1}_{m \notin \mathcal{M}_n(t-1)}  
\]
\STATE $S\gets\emptyset$, $B\gets\mu_n$
\WHILE{$\exists m\in\mathcal{M}\setminus S$ with $s_m\le B$}
    \STATE $\mathcal{C}\gets\{m\notin S: s_m\le B\}$
    \STATE $m^\star\gets\arg\max_{m\in\mathcal{C}}
    \frac{\Delta U_n(m\mid S)}{s_m}$,
    \quad $\Delta U_n(m\mid S)=U_n(S\cup\{m\})-U_n(S)$
    \IF{$\Delta U_n(m^\star\mid S)<0$}
        \STATE \textbf{break}
    \ENDIF
    \STATE $S\gets S\cup\{m^\star\}$,\quad $B\gets B-s_{m^\star}$
\ENDWHILE
\STATE $\mathcal{M}_n(t)\gets S$
\end{algorithmic}
\end{algorithm}

\section{Additional Experimental Details and Results}\label{app:exp}
The experimental code and dataset are provided \href{https://anonymous.4open.science/r/llmooo-572B/README.md}{here}.  
\subsection{Dataset Details}\label{app:dataset_details}
To evaluate the proposed HI routing methods as well as the greedy model placement strategy, we collect LLMs inference results on various benchmark datasets from RouterBench~\citep{hu2024routerbench} and VL-RouterBench~\citep{huang2025vl}. 
We select 23 models, including text-only language models, vision-language, and multimodal foundation models spanning a wide range of architectures, parameter scales, and capability levels. The complete model set is listed below:
\begin{itemize}
    \item \textbf{GPT-2 family~\citep{radford2019language}:}
    GPT2-Large,
    GPT2-XL.
    
    \item \textbf{DeepSeek family~\citep{guo2025deepseek}:}
    DeepSeek-LLM-7B,
    DeepSeek-MoE-16B,
    DeepSeek-LLM-67B,
    DeepSeek-VL2,
    DeepSeek-VL2-Tiny.
    
    \item \textbf{Qwen family~\citep{bai2023qwen}:}
    Qwen2-0.5B,
    Qwen1.5-0.5B,
    Qwen2-72B,
    Qwen2.5-VL-32B-Instruct,
    Qwen2.5-VL-72B-Instruct.
    
    \item \textbf{Gemma~\citep{team2024gemma}:}
    Gemma3-27B.
    
    \item \textbf{Vision--Language / Multimodal models:}
    InternVL2.5-78B,
    Janus-Pro-1B,
    Janus-Pro-7B,
    Kimi-VL-A3B-Thinking,
    MiMo-VL-7B-RL,
    Phi-3.5-Vision,
    Pixtral-12B,
    Qianfan-VL-8B,
    SmolVLM2,
    LLaVA-NeXT-Vicuna-7B.
\end{itemize}
The smallest model in this suite is Qwen2-0.5B (and Qwen1.5-0.5B), which typically requires around 1--2\,GB of GPU memory for inference deployment (FP16), making it well suited for edge devices and lower-level computing nodes. The largest model is InternVL2.5-78B, which requires on the order of 150--160\,GB of GPU memory for full-precision deployment, and thus is only deployable for higher-level nodes equipped with substantial computational resources. Such large-capacity models are deployed selectively to process difficult tasks routed upward from lower layers in the hierarchy.

Following \cite{hu2024routerbench} and \cite{huang2025vl}, we include multimodal reasoning datasets (MMMU~\citep{yue2024mmmu}, MMBench~\citep{liu2024mmbench}, MathVista~\citep{lu2023mathvista}), document and vision-language benchmarks (DocVQA~\citep{mathew2021docvqa}, TextVQA~\citep{singh2019towards}, ChartQA~\citep{masry2022chartqa}), mathematical reasoning benchmarks (GSM8K~\citep{cobbe2021gsm8k}, MATH~\citep{hendrycksmath2021}), and general reasoning and knowledge benchmarks (BBH~\citep{suzgun2022challenging}, MMLU~\citep{hendryckstest2021}, GPQA~\citep{rein2024gpqa}, ARC-Challenge~\citep{clark2018think}, HellaSwag~\citep{zellers2019hellaswag}, Winogrande~\citep{sakaguchi2021winogrande}). 
We define task types $y\in\mathcal{Y}$ based on datasets: each dataset corresponds to at least one task type, and datasets containing multiple subtasks are further decomposed into multiple task types to capture heterogeneous difficulty and capability requirements. Overall, we group 77,988 inference job samples into 114 task types.

\subsection{System Setup Details}\label{app:setup_details}
\begin{table}[h]
\centering
\caption{System constraints across hierarchical layers. The memory budget increases with the layer depth, while the oracle layer has no memory constraint. Each non-entry node has a resource constraint of $0.4$.}
\begin{tabular}{@{}l|ccc@{}}
\toprule
\multirow{2}{*}{Configuration} & \multicolumn{3}{c}{Hierarchy Settings} \\ \cmidrule(l){2-4}
 & 3-layer (4-2-1) & 4-layer (8-4-2-1) & 5-layer (16-8-4-2-1) \\ \midrule
Memory Budget $\mu_n$ 
& (30, 100, --) 
& (30, 80, 200, --) 
& (30, 80, 150, 200, --) \\

Resource Constraint $\gamma_n$ 
& (--, 0.4, 0.4) 
& (--, 0.4, 0.4, 0.4) 
& (--, 0.4, 0.4, 0.4, 0.4) \\ 
\bottomrule
\end{tabular}
\end{table}
Each node is assigned a memory budget $\mu_n$ as specified above. We set the exploration rate $\lambda = 0.1$, drift-plus-penalty factor $v=70$, and downloading penalty $\nu=0.1$ in the experiments. 
% The large value of the drift-plus-penalty factor $v=70$ is mainly due to the different scales of the routing cost and the inference error as we model.

The offloading cost $c^j(n,n')$ is determined by the input size of a job (e.g., text length or image size) and the distance between nodes. We model one unit of routing cost corresponds to transmitting approximately 10KB characters. Vision tasks incur significantly higher costs, typically exceeding 100KB (i.e., over 10 cost units). Without careful offloading decisions, the routing budget can be rapidly exhausted, causing most jobs to be processed locally. 

\subsection{Baseline Details}\label{app:baseline_details}
\textbf{Pure Local:} Under this deterministic policy, the offloading probability is zero. All inference jobs are forced to execute and terminate locally at the initial entry node $n\in\mathcal{N}_1$.

\textbf{Random:} This baseline constructs a time-invariant probability distribution for routing actions. At any given node $n$, the probability of offloading a job $j$ to an upstream destination $a$, denoted by the probability $p_n^j(a)$, is fixed from a static distribution. The routing destination is selected randomly such that the aggregate offloading probability across all upstream nodes adheres to the long-term resource constraints defined for the system across all nodes.

\textbf{Round-Robin:} This strategy operates via a two-step routing sequence. First, the node uses a fixed probability distribution $\{p_n^j(0), 1-p_n^j(0)\}$ to determine whether the inference process should terminate locally or offload. 
If the realized decision dictates offloading, the specific destination node is selected sequentially from the available upstream set $\mathcal{U}_n$ in a round-robin manner. 
This alternation ensures an equitable distribution of offloaded jobs across all available upstream destinations.

\textbf{Calibration of Routing Probabilities:} To ensure a fair comparison, the routing probabilities for the static baselines should satisfy the node-level long-term resource constraints $\gamma_n$ across the system. 
We firstly compute the expected resource cost of routing all jobs per node with respect to its job arrival distribution. 
Because the network topology is constructed such that the number of nodes doubles at each successive layer, and every intermediate node of layer $k>1$ is restricted by a uniform resource budget of $\gamma_n = 0.4$, the ``effective'' allowable resource consumption from each node of layer $k-1$ is $0.2$.
By dividing this allocated budget by the expected total resource cost per node, we derive the offloading probability required to maintain long-term resource feasibility within the system. This calculation yields an average offloading probability of $0.12$ (representing the sum $\sum_{a \in \mathcal{U}_n} p_n^j(a)$) for \emph{Random} and \emph{Round-Robin} baseline methods.

\textbf{Comparison of Resource Utilization.}
We validate that these calibrated probabilities satisfy the resource constraints by tracking the average routing costs across all nodes. 
The average resource consumption for each routing policy conforms to the theoretical limits. The empirical average routing cost for each method is recorded as follows:
0.4269 (Ly-EXP4), 
0.3816 (VR-Ly-EXP4), 
0.3928 (VR-Ly-EXP4-LocalLoss), 
0.3965 (Random),
0.3842 (Round-Robin).

\subsection{Additional Results}\label{app:extra_results}

\textbf{Entropy of expert weights.}
The entropy of the expert weight distribution serves as a quantitative measure of policy convergence; a more rapid rate of decay in entropy indicates accelerated learning and faster policy stabilization. 
As shown in Figure \ref{fig:entropy},
the integration of the variance reduction mechanism, as implemented in both VR-Ly-EXP4 and VR-Ly-EXP4-LocalLoss, substantially accelerates this entropy decay, thereby facilitating faster policy learning compared to standard estimator Ly-EXP4. 
Furthermore, incorporating the recursive expected loss from uplink nodes ($\overline{f}_n(j,\pi_n)$) enhances the learning quality, particularly during the initial exploration phase. This recursive formulation ensures that each intermediate node can accurately model and evaluate the true routing loss across distinct uplink destinations, rather than optimizing based solely on localized offloading costs.
\begin{figure}
    \centering
    \includegraphics[width=0.5\linewidth]{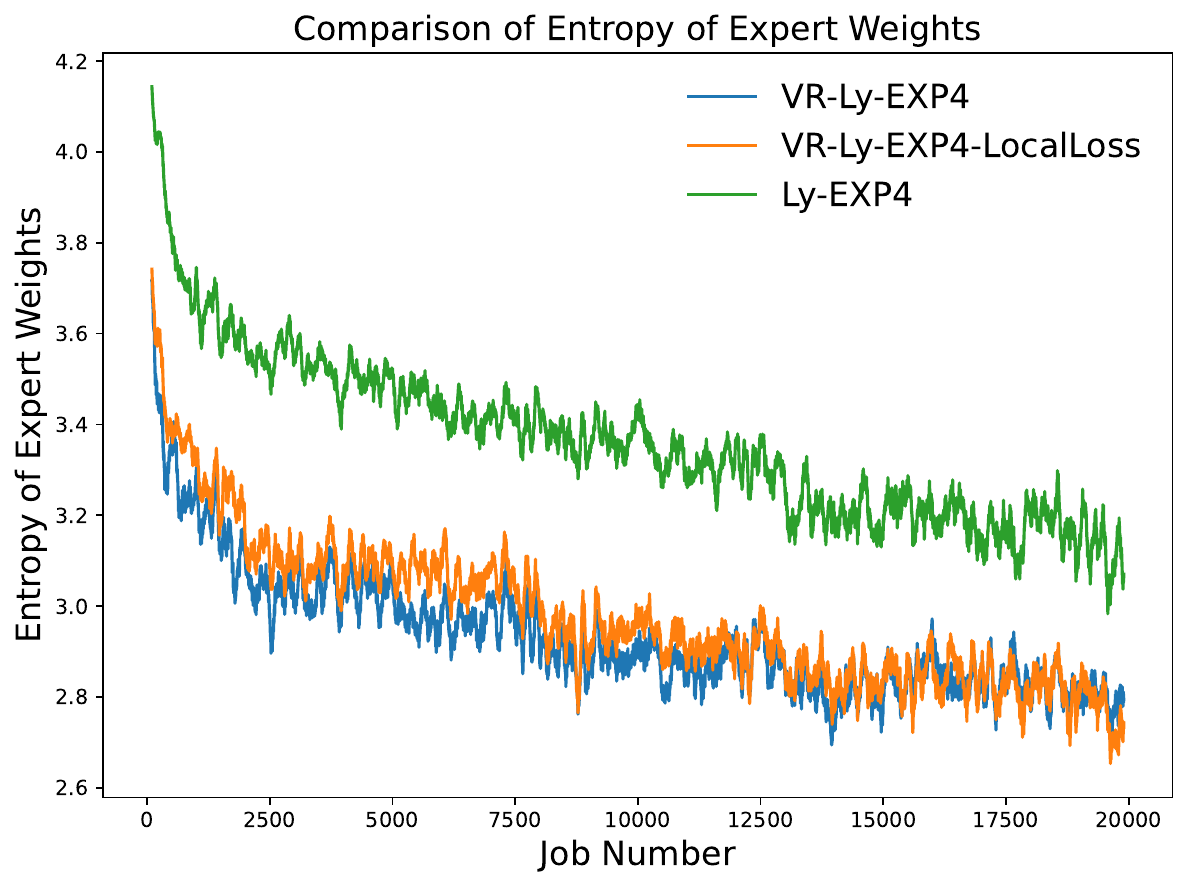}
    \caption{3-layer hierarchy (4-2-1). Entropy of expert weights, averaged across all nodes.}
    \label{fig:entropy}
\end{figure}

% \textbf{Average Error Plots}
% ``Optimal'' offloading without constraints avg error: 0.1201.

\textbf{Ablation on Model Placement Strategies.}
We compare the proposed greedy model placement with two non-adaptive baselines:

\textit{(1) Random-fixed.}
Each node randomly selects models until its memory budget is filled and then keeps the placement fixed throughout training. This strategy ignores workload dynamics and system state.

\textit{(2) Layer-diverse.}
The model pool is partitioned into $K$ disjoint groups, and each layer is assigned a distinct subset. While this enforces inter-layer diversity, it does not adapt placement to the task distribution or network conditions.

Table~\ref{tab:onloading} reports the error rates under different routing algorithms. The greedy placement (Algorithm \ref{alg:greedy_onloading}) consistently improves performance for adaptive routing methods (Ly-EXP4 and VR-Ly-EXP4), achieving the lowest error overall. In contrast, static placements (Random-fixed and Layer-diverse) fail to adapt to workload variations, leading to higher error. Notably, although Greedy performs worse under purely random routing due to reduced exploration diversity, it becomes clearly superior once learning-based routing is enabled, demonstrating the importance of adaptive model placement.

\begin{table}[h]
\centering
\caption{Model placement comparison (3-layer hierarchy). Error rates under different routing and model placement strategies.}
\label{tab:onloading}
\begin{tabular}{@{}l|ccc@{}}
\toprule
\multirow{2}{*}{Model Placement} & \multicolumn{3}{c}{Error Rate ($\downarrow$)} \\ \cmidrule(l){2-4}
 & Random Routing & Ly-EXP4 & VR-Ly-EXP4 \\ \midrule
Random-fixed   & \textbf{0.4332} & 0.3717 & 0.3335 \\
Layer-diverse  & 0.4455 & 0.3928 & 0.3477 \\
Greedy (Algorithm \ref{alg:greedy_onloading})        & 0.4805 & \textbf{0.3433} & \textbf{0.3172} \\
\bottomrule
\end{tabular}
\end{table}

\newpage

\section{Theoretical Analysis}
\label{app:proofs}

% For each node $n \in \mathcal{N}\setminus \mathcal{N}_{1}$, we define a virtual queue $Q_{n}(t)$ that tracks the deviation between instantaneous resource consumption and the per-slot budget $\gamma_{n}\tau$. 
% Define $Q_{n}(0)=0$. 
% The queue evolves as
% \begin{align}
% Q_{n}(t+1) = \max\left\{Q_{n}(t) + C^{\boldsymbol{\pi}}_{n}(t) - \gamma_{n} \tau, \, 0\right\}.
% \end{align}
% The virtual queue tracks the accumulated debt of processing jobs offloaded from lower-layer nodes. Each node maintains its own virtual queue, except for the entry nodes in the initial layer $\mathcal{N}_1$. 
% We assume the exogenous job arrivals at the initial layer $\mathcal{N}_1$ have a bounded second moment, i.e., $\mathbb{E}[|\mathcal{J}(t)|^2] \le A_{max}$. This strict upper bound is theoretically necessary for two reasons: (1) it guarantees the finiteness of the Lyapunov drift constant $B$ by capping the maximum per-slot offloading cost $C^{\boldsymbol{\pi}}_{n}(t)$, and (2) it ensures the system satisfies Slater's condition (strict feasibility), which is fundamentally required to stabilize the distributed virtual queues.

\subsection{Decomposition of the Task-Conditioned Baseline} \label{app:baseline_compute}

we first establish theoretical formulation for the baseline loss $\bar{f}_{n, y}^{n'}(\pi_h)=\mathbb{E}_{\tilde{j} | y(\tilde{j})=y}[f_{n}^{n'}(\tilde{j},\pi_h)]$ and then outline the practical approximation utilized in our experiments.

We unfold the conditional expectation of the loss for a random future job $\tilde{j}$ of task type $y$. By treating the system state variables $q_{n'}(t)$ and $\mathbf{x}_n(t)$ as fixed during time slot $t$, the randomness is strictly governed by the task arrival distribution. 

Using the deterministic loss formulation, the theoretical baseline is expanded using the law of total expectation conditioned on the threshold event:
\begin{align}
\mathbb{E}_{J \sim \mathcal{I}_n(t) | y(J)=y}[f_{n}^{n'}(J,\pi_h)] &= v \cdot \mathbb{P}(Z_n(J) \geq \theta_h | y(J)=y) \cdot \mathbb{E}[B(J, \mathbf{x}_n(t)) | Z_n(J) \geq \theta_h, y(J)=y] \nonumber \\
&\quad + \mathbb{P}(Z_n(J) < \theta_h | y(J)=y) \cdot \Big( q_{n'}(t) \mathbb{E}[c^{J}(n,n') | Z_n(J) < \theta_h, y(J)=y] \nonumber \\
&\quad \quad \quad \quad \quad \quad \quad \quad \quad \quad + \mathbb{E}[\bar{f}_{n'}(J, \pi_{n'}) | Z_n(J) < \theta_h, y(J)=y] \Big). \label{eq:unfolded_baseline}
\end{align}

This decomposition provides critical theoretical insight into the learning dynamics. The baseline explicitly captures the correlation between the local confidence score $Z_n(\tilde{j})$ and the actual inference error $B(\tilde{j}, \mathbf{x}_n(t))$ for a specific task type. By tracking this expectation historically, the variance-reduced estimator effectively isolates the stochasticity of the partial feedback indicator $\mathbbm{1}_{\text{fb}}(\tilde{j})$ from the inherent variance of the workload difficulty.

While Eq. \ref{eq:unfolded_baseline} formalizes the exact theoretical expectation, computing this analytically is intractable in practical deployments due to the unknown joint distributions governing the task arrival dynamics $\mathcal{I}_n(t)$ and stochastic inference errors $B(J, \mathbf{x}_n(t))$.
Alternatively, we estimate the baseline empirically using an exponential moving average (with learning rate $\eta_b$), updated exclusively when a valid terminal feedback signal is observed (i.e., $\mathbbm{1}_{\text{fb},n}(j) = 1$):
\begin{align}
\bar{f}_{n, y}^{n'}(\pi_h) = 
\begin{cases} 
(1-\eta_b)\bar{f}_{n, y}^{n'}(\pi_h) + \eta_b \frac{f_n^{n'}(j,\pi_h)}{\rho_n^\pi(j)} & \text{if } \mathbbm{1}_{\text{fb},n}(j) = 1, \\ 
\bar{f}_{n, y}^{n'}(\pi_h) & \text{otherwise}.
\end{cases}
\end{align}
By explicitly dividing the realized loss by the marginal feedback probability $\rho_n^\pi(j)$, this update corrects for partial observability and ensures the baseline $\bar{f}_{n, y}^{n'}(\pi_h)$ acts as an unbiased estimator of the theoretical expectation.

Notice that the unbiasedness of the baseline itself does not affect the unbiasedness of the variance-reduced estimator $\hat{F}_{\text{vr}, n}^{n'}(j,\pi_h)$. 
Because the baseline term is subtracted within the importance-weighted fraction and subsequently added back outside of it, it mathematically cancels out when computing the expectation of the estimator.
This guarantees $\hat{F}_{\text{vr}, n}^{n'}(j,\pi_h)$ remains unbiased regardless of the empirical accuracy of the baseline $\bar{f}_{n, y}^{n'}(\pi_h)$.

\subsection{Proof of Proposition \ref{prop:queue}}\label{app:proposition}
\begin{proof}
By the definition of the virtual queue evolution, we have:
\begin{align}
Q_{n}(t+1) = \max\left\{Q_{n}(t) + C^{\boldsymbol{\pi}}_{n}(t) - \gamma_{n} \tau, \, 0\right\}.
\end{align}

From the property of the maximum function, this strictly implies:
\begin{align}
Q_{n}(t+1) - Q_{n}(t) \ge C^{\boldsymbol{\pi}}_{n}(t) - \gamma_{n} \tau.
\end{align}

Summing this inequality over $t = 1, \dots, T$ yields a telescoping sum on the left-hand side:
\begin{align}
Q_{n}(T+1) - Q_{n}(1) \ge \sum_{t=1}^T \left(C^{\boldsymbol{\pi}}_{n}(t) - \gamma_{n} \tau\right).
\end{align}

Dividing both sides by $T$ and taking the expectation $\mathbb{E}[\cdot]$ gives:
\begin{align}
\frac{\mathbb{E}[Q_{n}(T+1)]}{T} - \frac{\mathbb{E}[Q_{n}(1)]}{T} \ge \frac{1}{T} \sum_{t=1}^T \mathbb{E}[C^{\boldsymbol{\pi}}_{n}(t)] - \gamma_{n} \tau.
\end{align}

Taking the limit as $T \to \infty$ and noting that the initial queue state is zero ($Q_{n}(1) = 0$), we obtain:
\begin{align}
\lim_{T \to \infty} \frac{\mathbb{E}[Q_{n}(T+1)]}{T} \ge \lim_{T \to \infty} \frac{1}{T} \sum_{t=1}^T \mathbb{E}[C^{\boldsymbol{\pi}}_{n}(t)] - \gamma_{n} \tau.
\end{align}

Given that the queue $Q_{n}(t)$ is mean-rate stable, the limit on the left-hand side evaluates exactly to $0$. Rearranging the inequality yields:
\begin{align}
\lim_{T \to \infty} \frac{1}{T} \sum_{t=1}^T \mathbb{E}[C^{\boldsymbol{\pi}}_{n}(t)] \le \gamma_{n} \tau.
\end{align}

This perfectly recovers the long-term resource constraint defined in \eqref{eq:constraint_offload}, completing the proof.
\end{proof}

\subsection{Detailed Steps to Derive the Drift Bound (\ref{eq:driftbound})} \label{app:driftbound}

To ensure all virtual queues are mean-rate stable, we define a quadratic Lyapunov function tracking the resource utilization across all intermediate and cloud nodes:
\begin{align}
\mathcal{L}(t) = \frac{1}{2} \sum_{n \in \mathcal{N} \setminus \mathcal{N}_1} \left(Q_{n}(t)\right)^2.
\end{align}

If we can show the Lyapunov function is stable across time, then all virtual queues are mean-rate stable. To show this, we analyze the single-slot conditional Lyapunov drift $\Delta(\boldsymbol{q}(t)) = \mathbb{E}[\mathcal{L}(t+1) - \mathcal{L}(t) \mid \boldsymbol{Q}(t) = \boldsymbol{q}(t)]$, where $\boldsymbol{Q}(t)$ is the random vector of all virtual queues and $\boldsymbol{q}(t) = \{q_n(t)\}_{n \in \mathcal{N} \setminus \mathcal{N}_1}$ is its deterministic realization at the beginning of slot $t$. Squaring the queue evolution equation yields:
\begin{align}
\mathcal{L}(t+1) - \mathcal{L}(t) 
&= \frac{1}{2} \sum_{n \in \mathcal{N} \setminus \mathcal{N}_1} \left[ \left(Q_{n}(t+1)\right)^2 - \left(Q_{n}(t)\right)^2 \right] \\
&\le \frac{1}{2} \sum_{n \in \mathcal{N} \setminus \mathcal{N}_1} \left[ \left(Q_{n}(t) + C^{\boldsymbol{\pi}}_{n}(t) - \gamma_{n}\tau\right)^2 - \left(Q_{n}(t)\right)^2 \right] \label{eq:max_inequality} \\
&= \frac{1}{2} \sum_{n \in \mathcal{N} \setminus \mathcal{N}_1} \left[ \left(C^{\boldsymbol{\pi}}_{n}(t) - \gamma_{n}\tau\right)^2 + 2Q_{n}(t)\left(C^{\boldsymbol{\pi}}_{n}(t) - \gamma_{n}\tau\right) \right] \\
&= \frac{1}{2} \sum_{n \in \mathcal{N} \setminus \mathcal{N}_1} \left(C^{\boldsymbol{\pi}}_{n}(t) - \gamma_{n}\tau\right)^2 + \sum_{n \in \mathcal{N} \setminus \mathcal{N}_1} Q_{n}(t)\left(C^{\boldsymbol{\pi}}_{n}(t) - \gamma_{n}\tau\right).
\end{align}
Note that inequality \eqref{eq:max_inequality} strictly holds due to the mathematical property $(\max\{0, x\})^2 \le x^2$. 

Taking the conditional expectation given the realized queue state $\boldsymbol{Q}(t) = \boldsymbol{q}(t)$ bounds the drift:
\begin{align}
\Delta(\boldsymbol{q}(t)) \le \alpha + \sum_{n \in \mathcal{N} \setminus \mathcal{N}_1} q_{n}(t) \mathbb{E}\left[C^{\boldsymbol{\pi}}_{n}(t) - \gamma_{n}\tau|\mathbf{Q}(t)=\mathbf{q}(t) \right],
\end{align}
where $B$ is strictly defined as a finite constant upper bounding the variance of the queue updates:
\begin{align}
\alpha \triangleq \sup_t \frac{1}{2} \sum_{n \in \mathcal{N} \setminus \mathcal{N}_1} \mathbb{E}\left[ \left(C^{\boldsymbol{\pi}}_{n}(t) - \gamma_{n}\tau\right)^2 |\mathbf{Q}(t)=\mathbf{q}(t)\right].
\end{align}
Because the random exogenous job arrivals $\mathcal{J}(t)$ at the initial layer have a strictly bounded second moment ($\mathbb{E}[|\mathcal{J}(t)|^2] \le A_{max}$), and the deterministic single-hop transmission cost $c^j(n', n)$ is bounded, the maximum possible accumulated cost $C_{n}^{\boldsymbol{\pi}}(t)$ per slot is strictly finite. Consequently, $\alpha$ is a strictly positive, finite constant.

\subsection{Proof of Lemma \ref{lemma:variance}}\label{app:variance}
\begin{proof}
To streamline the analysis for a given node $n\in\mathcal{N}_k$, upstream node $n'\in\mathcal{N}_{k+1}$, a specific realized job $j$ with task type $y$, and an expert policy $\pi_h$, we define the following deterministic variables:
\begin{itemize}
    \item Let $x_j = f_n^{n'}(j,\pi_h)$ denote the true ideal loss for job $j$.
    \item Let $\beta = \bar{f}_{n,y}^{n'}(\pi_h)$ denote the baseline loss estimator computed from historical observations.
    \item Let $\rho_j = \rho_n^{\boldsymbol{\pi}}(j) \in (0,1]$ represent the marginal routing probability that job $j$ reaches the final oracle layer, such that $\mathbb{E}[\mathbbm{1}_{\text{fb},n}(j)] = \rho_j$.
\end{itemize}

The naive estimator and the variance-reduced estimator are defined as the random variables:
\begin{align}
\hat{F}_{\text{naive},n}^{n'}(j,\pi_h) &= \mathbbm{1}_{\text{fb},n}(j)\frac{x_j}{\rho_j} \\
\hat{F}_{\text{vr},n}^{n'}(j,\pi_h) &= \mathbbm{1}_{\text{fb},n}(j)\frac{x_j-\beta}{\rho_j} + \beta
\end{align}
The only randomness comes from the feedback indicator $\mathbbm{1}_{\text{fb},n}$.
We evaluate the variance of both estimators strictly with respect to the Bernoulli random variable $\mathbbm{1}_{\text{fb},n}(j) \sim \text{Bernoulli}(\rho_j)$. The exact variance of a scaled Bernoulli random variable $c \cdot \mathbbm{1}_{\text{fb},n}(j)$, where $c$ is a deterministic constant, is given by $c^2 \rho_j (1 - \rho_j)$.

First, we calculate the exact variance of the naive estimator:
\begin{equation}
Var\left(\hat{F}_{\text{naive},n}^{n'}(j,\pi_h)\right) = \left(\frac{x_j}{\rho_j}\right)^2 Var(\mathbbm{1}_{\text{fb},n}(j)) = \frac{x_j^2}{\rho_j^2} \rho_j (1 - \rho_j) = x_j^2 \frac{1-\rho_j}{\rho_j}
\end{equation}

Next, we calculate the exact variance of the variance-reduced estimator. Because $\beta$ is a deterministic constant for a realized job $j$, it does not affect the variance calculation:
\begin{equation}
Var\left(\hat{F}_{\text{vr},n}^{n'}(j,\pi_h)\right) = Var\left(\mathbbm{1}_{\text{fb},n}(j)\frac{x_j-\beta}{\rho_j}\right) = \left(\frac{x_j-\beta}{\rho_j}\right)^2 \rho_j (1 - \rho_j) = (x_j-\beta)^2 \frac{1-\rho_j}{\rho_j}
\end{equation}

Let $\Delta_{\text{Var}}(j)$ represent the exact reduction in variance for the specific job $j$. We expand the squares fully to show the precise difference:
\begin{align}
\Delta_{\text{Var}}(j) &= Var\left(\hat{F}_{\text{naive},n}^{n'}(j,\pi_h)\right) - Var\left(\hat{F}_{\text{vr},n}^{n'}(j,\pi_h)\right) \\
&= x_j^2 \frac{1-\rho_j}{\rho_j} - (x_j-\beta)^2 \frac{1-\rho_j}{\rho_j} \\
&= \frac{1-\rho_j}{\rho_j} \left[ x_j^2 - (x_j^2 - 2x_j\beta + \beta^2) \right] \\
&= \frac{1-\rho_j}{\rho_j} (2x_j\beta - \beta^2) \\
&= \beta(2x_j - \beta)\frac{1-\rho_j}{\rho_j}
\end{align}

Since $\rho_j$ represents the routing probability along the upstream path, it strictly satisfies $\rho_j \in (0,1]$, which mathematically guarantees that $\frac{1-\rho_j}{\rho_j} \geq 0$.

By the prerequisite condition established in the lemma, the baseline loss satisfies $0 < \beta \leq 2x_j$. Therefore, we have strict positivity for the baseline $\beta > 0$ and non-negativity for the residual $(2x_j - \beta) \geq 0$. Consequently, their product is non-negative, ensuring:
\begin{equation}
\Delta_{\text{Var}}(j) \geq 0
\end{equation}

This strictly guarantees:
\begin{equation}
Var\left(\hat{F}_{\text{vr},n}^{n'}(j,\pi_h)\right) \leq Var\left(\hat{F}_{\text{naive},n}^{n'}(j,\pi_h)\right)
\end{equation}
This completes the proof.
\end{proof}

\subsection{Proof of Theorem \ref{theorem:regret}}

\textbf{Definition 1 (History Filtration).} Let $\pi$ denote the joint routing policy across all computing nodes. The system history up to the end of time slot $t-1$ is captured by the filtration $\mathcal{F}_{t-1}$, defined as the $\sigma$-algebra generated by all prior stochastic job arrivals, system states, routing decisions, and partially observed feedback:
\begin{align}
\mathcal{F}_{t-1} = \sigma\Big( \big\{ \mathcal{J}(t'), \mathbf{Q}(t'), \{\mathbf{x}_n(t')\}_{n\in\mathcal{N}}, \{\mathbf{O}_n^j\}_{n\in\mathcal{N},j\in\mathcal{J}(t')}: 1 \leq t' \leq t-1 \big\} \Big).
\end{align}
$\mathcal{F}_{t-1}$ tracks the history of the system. 
We define conditional expectation $\mathbb{E}_{t-1}[\cdot]=\mathbb{E}[\cdot|\mathcal{F}_{t-1}]$.

To establish the regret bound of the proposed variance-reduced EXP4 algorithm in the arbitrary-depth multi-node setting, we first rigorously prove the unbiasedness of the loss estimator and the boundedness of the multi-tier virtual queues. 

\begin{lemma}[Unbiasedness of the Estimator]\label{lemma:unbiasedness}
For any intermediate node $n \in \mathcal{N} \setminus \mathcal{N}_K$, its upstream node $u\in\mathcal{U}_n$, and a specific job $j$, the random variance-reduced estimator $\hat{F}_{\text{vr}, n}^u(j, \pi_h)$ is an unbiased estimator of the deterministic true joint expert loss $f_n^u(j, \pi_h)$. Specifically:
\begin{align}
\mathbb{E}_{t-1}\left[\hat{F}_{\text{vr}, n}^u(j, \pi_h)\right] = f_n^u(j, \pi_h)
\end{align}
\end{lemma}

\begin{proof}[Proof of Lemma \ref{lemma:unbiasedness}]
By definition of the variance-reduced estimator under partial feedback:
\begin{align}
\hat{F}_{\text{vr}, n}^u(j, \pi_h) = \mathbbm{1}_{\text{fb},n}(j) \frac{f_n^u(j, \pi_h) - \bar{f}_{n,y(j)}^u(\pi_h)}{\rho_n^{\boldsymbol{\pi}}(j)} + \bar{f}_{n,,y(j)}^u(\pi_h)
\end{align}
Taking the conditional expectation with respect to the history $\mathcal{F}_{t-1}$, the only random variable specific to the feedback mechanism is the binary indicator $\mathbbm{1}_{\text{fb},n}(j)$, which evaluates to $1$ if the job visited node $n$ and reached the final oracle layer, and $0$ otherwise. By definition, we have $\mathbb{E}_{t-1}[\mathbbm{1}_{\text{fb},n}(j)] = \rho_n^{\boldsymbol{\pi}}(j)$.
\begin{align}
\mathbb{E}_{t-1}\left[\hat{F}_{\text{vr}, n}^u(j, \pi_h)\right] &= \mathbb{E}_{t-1}[\mathbbm{1}_{\text{fb},n}(j)] \frac{f_n^u(j, \pi_h) - \bar{f}_{n,,y(j)}^u(\pi_h)}{\rho_n^{\boldsymbol{\pi}}(j)} + \bar{f}_{n,,y(j)}^u(\pi_h) \nonumber \\
&= \rho_n^{\boldsymbol{\pi}}(j) \frac{f_n^u(j, \pi_h) - \bar{f}_{n,,y(j)}^u(\pi_h)}{\rho_n^{\boldsymbol{\pi}}(j)} + \bar{f}_{n,,y(j)}^u(\pi_h) \nonumber \\
&= f_n^u(j, \pi_h) - \bar{f}_{n,,y(j)}^u(\pi_h) + \bar{f}_{n,,y(j)}^u(\pi_h) \nonumber \\
&= f_n^u(j, \pi_h)
\end{align}
This explicitly confirms strict unbiasedness.
\end{proof}

\begin{lemma}[Bounded Moments for Multi-Tier Queues]\label{lemma:bounded_moments}
Assume Slater's condition strictly holds with slack $\epsilon > 0$, guaranteeing strict feasibility of the shared network capacity. Let $C_{min}$ be the deterministic lower bound on resource costs. For the trade-off parameter $v > 0$, the random virtual queue $Q_n(t)$ at any node $n \in \mathcal{N} \setminus \mathcal{N}_1$ possesses bounded first and second moments:
\begin{align}
\mathbb{E}[Q_n(t)] \le \mathcal{O}\left(\frac{v}{C_{min}}\right), \quad \mathbb{E}[(Q_n(t))^2] \le \mathcal{O}\left(\left(\frac{v}{C_{min}}\right)^2\right)
\end{align}
\end{lemma}

\begin{proof}[Proof of Lemma \ref{lemma:bounded_moments}]
We follow the proof steps of \cite{beytur2024optimization} described in its Eq. (39)--(47).
Because decisions across multi-tier nodes are coupled via the drift-plus-penalty framework, any random queue state $Q_n(t)$ growing sufficiently large strictly dominates the local objective $Q_n(t)C_n^{\boldsymbol{\pi}}(t)$.
We define the deterministic critical threshold $l^* = \frac{v}{C_{min}}$. When the realized queue state exceeds this threshold ($Q_n(t) \ge l^*$), the random cost penalty associated with utilizing that upstream node strictly satisfies:
\begin{align}
Q_n(t) C_n^{\boldsymbol{\pi}}(t) \ge \left(\frac{v}{C_{min}}\right) C_{min} = v
\end{align}
Because the deterministic accuracy penalty is strictly bounded by $v$ (since inference error $b(j, \mathbf{x}_n(t)) \le 1$), the queue congestion cost strictly outweighs any potential accuracy gain. Consequently, the exponential weight updates in EXP4 force the routing probability to that specific congested node to decay exponentially, creating a negative drift condition for the conditional expectation:
\begin{align}
\mathbb{E}_{\mathcal{F}_{t-1}}\left[ Q_n(t+1) - Q_n(t) \mid Q_n(t) \ge l^* \right] \le -\epsilon
\end{align}
By applying the Chernoff bound to the moment generating function of the queue evolution, the tail probability decays exponentially $\mathbb{P}(Q_n(t) \ge q) \le K \exp(\eta_0(l^* - q))$ for some deterministic constants $K, \eta_0 > 0$. Integrating this tail probability strictly yields the bounded second moment:
\begin{align}
\mathbb{E}[(Q_n(t))^2] &= \int_{0}^{\infty} 2q \mathbb{P}(Q_n(t) \ge q) dq \nonumber \\
&\le \mathcal{O}((l^*)^2) = \mathcal{O}\left(\left(\frac{v}{C_{min}}\right)^2\right)
\end{align}
This concludes the proof.
\end{proof}

With the above Lemmas, we proceed to provide the \textbf{proof of Theorem \ref{theorem:regret}}.

\begin{proof}[Proof of Theorem \ref{theorem:regret}]
We employ the potential function method over the joint decision space $\mathcal{E}_n(y)$ to bound the cumulative regret for task type $y$. To strictly ensure the non-negativity required by the exponential inequalities without altering the algorithm's actual routing distribution, we define a step-wise, expert-independent translation.

\textbf{Step 1: Shift-Invariant Loss and Unbiasedness.} 
For any job $j \in \mathcal{J}(y)$, let $\Upsilon_j$ be the maximum possible negative deviation across all experts in the joint space, realized only if the job is routed to node $n$:
\begin{align}
\Upsilon_j = \mathbbm{1}_{n \in \Omega^{\boldsymbol{\pi},j}} \max_{(h,u) \in \mathcal{E}_n(y)} \{ 0, - \hat{F}_{\text{vr}, n}^u(j, \pi_h) \}
\end{align}
We define the strictly non-negative shifted estimator as:
\begin{align}
\tilde{F}_{\text{vr}, n}^u(j, \pi_h) = \hat{F}_{\text{vr}, n}^u(j, \pi_h) + \Upsilon_j
\end{align}
By definition, $\tilde{F}_{\text{vr}, n}^u(j, \pi_h) \ge 0$ for all $(h,u) \in \mathcal{E}_n(y)$. Because $\Upsilon_j$ is identical for all experts at step $j$, the standard EXP4 exponential weight updates using $\tilde{F}_{\text{vr}}$ yield the exact same probability distribution $w_{n,y}^{h,u}(t^j)$ as using $\hat{F}_{\text{vr}}$, making the algorithm strictly shift-invariant.

Let $(h^*, u^*) \in \mathcal{E}_n(y)$ be the optimal fixed expert in hindsight. Because our variance-reduced estimator is strictly unbiased ($\mathbb{E}[\hat{F}_{\text{vr}, n}^u(j, \pi_h)] = f_{n}^{u}(j, \pi_h)$), we decompose the expected regret $R_{n,y}(J)$ into the realized estimators:
\begin{align}
R_{n, y}(\Gamma) &= \sum_{j=1}^\Gamma\mathbb{E} \left[  \mathbbm{1}_{y(j)=y} \mathbbm{1}_{n \in \Omega^{\boldsymbol{\pi},j}} F_n(j,\pi_n) \right] 
- \min_{(h, n') \in \mathcal{E}_{n}(y)} \sum_{j =1}^\Gamma\mathbb{E} \left[  \mathbbm{1}_{y(j)=y} \mathbbm{1}_{n \in \Omega^{\boldsymbol{\pi}, j}} f_{n}^{n'}(j, \pi_h) \right].
\end{align}

\textbf{Step 2: Evolution of the Potential Function.} 
We define $\mathcal{J}(y)$ as the set of jobs of type $y$ among all $\Gamma$ jobs. 
Let the jobs in $\mathcal{J}(y)$ be ordered temporally. We define the potential function after processing up to job $j$ as:
\begin{align}
W_j = \sum_{(h,u) \in \mathcal{E}_n(y)} \exp\left(-\eta \tilde{L}_{n,y}^{h,u}(j)\right)
\end{align}
where $\tilde{L}_{n,y}^{h,u}(j) = \sum_{j' \le j, j' \in \mathcal{J}(y)} \mathbbm{1}_{n \in \Omega^{\boldsymbol{\pi},j'}} \tilde{F}_{\text{vr}, n}^u(j', \pi_h)$ is the cumulative shifted loss. We analyze the ratio of potentials between consecutive tasks $j$ and $j-1$ in $\mathcal{J}(y)$:
\begin{align}
\frac{W_j}{W_{j-1}} &= \sum_{(h,u) \in \mathcal{E}_n(y)} \frac{\exp\left(-\eta \tilde{L}_{n,y}^{h,u}(j-1)\right)}{W_{j-1}} \exp\left(-\eta \mathbbm{1}_{n \in \Omega^{\boldsymbol{\pi},j}} \tilde{F}_{\text{vr}, n}^u(j, \pi_h)\right)
\end{align}
Recognizing the normalized expert weight $w_{n,y}^{h,u}(t^j)$, we substitute it:
\begin{align}
\frac{W_j}{W_{j-1}} = \sum_{(h,u) \in \mathcal{E}_n(y)} w_{n,y}^{h,u}(t^j) \exp\left(-\eta \mathbbm{1}_{n \in \Omega^{\boldsymbol{\pi},j}} \tilde{F}_{\text{vr}, n}^u(j, \pi_h)\right)
\end{align}

\textbf{Step 3: Bounding the Step Ratio.} 
Because $\tilde{F}_{\text{vr}, n}^u(j, \pi_h) \ge 0$, we rigorously apply the algebraic inequality $e^{-x} \le 1 - x + \frac{x^2}{2}$:
\begin{align}
\frac{W_j}{W_{j-1}} \le \sum_{(h,u) \in \mathcal{E}_n(y)} w_{n,y}^{h,u}(t^j) \Bigg( &1 - \eta \mathbbm{1}_{n \in \Omega^{\boldsymbol{\pi},j}} \tilde{F}_{\text{vr}, n}^u(j, \pi_h) \nonumber \\
&+ \frac{\eta^2}{2} \mathbbm{1}_{n \in \Omega^{\boldsymbol{\pi},j}} \left(\tilde{F}_{\text{vr}, n}^u(j, \pi_h)\right)^2 \Bigg)
\end{align}
Since the expert weights form a valid probability distribution, $\sum_{(h,u) \in \mathcal{E}_n(y)} w_{n,y}^{h,u}(t^j) = 1$. Applying this yields:
\begin{align}
\frac{W_j}{W_{j-1}} \le 1 &- \eta \mathbbm{1}_{n \in \Omega^{\boldsymbol{\pi},j}} \sum_{(h,u) \in \mathcal{E}_n(y)} w_{n,y}^{h,u}(t^j) \tilde{F}_{\text{vr}, n}^u(j, \pi_h) \nonumber \\
&+ \frac{\eta^2}{2} \mathbbm{1}_{n \in \Omega^{\boldsymbol{\pi},j}} \sum_{(h,u) \in \mathcal{E}_n(y)} w_{n,y}^{h,u}(t^j) \left(\tilde{F}_{\text{vr}, n}^u(j, \pi_h)\right)^2
\end{align}

\textbf{Step 4: Logarithmic Bounding and Telescoping Sum.} 
Let $\tilde{F}_{max} = \sup_{j, (h,u)} \tilde{F}_{vr,n}^u(j,\pi_h)$ denote the maximum possible realization of the shifted loss estimator. We operate under the standard learning-theoretic assumption that the learning rate is bounded such that $\eta \tilde{F}_{max} < 1$.

Because the expert weights form a valid probability distribution, this assumption strictly guarantees that:
\begin{equation}
1 - \eta \mathbb{I}_{n \in \Omega^{\pi,j}} \sum_{(h,u) \in \mathcal{E}_n(y)} w_{n,y}^{h,u}(t^j) \tilde{F}_{vr,n}^u(j,\pi_h) > 0.
\end{equation}

Taking the natural logarithm and using the strict inequality $\ln(1+x) \le x$, we obtain:
\begin{align}
\ln\left(\frac{W_j}{W_{j-1}}\right) &\le -\eta \mathbbm{1}_{n \in \Omega^{\boldsymbol{\pi},j}} \sum_{(h,u) \in \mathcal{E}_n(y)} w_{n,y}^{h,u}(t^j) \tilde{F}_{\text{vr}, n}^u(j, \pi_h) \nonumber \\
&\quad + \frac{\eta^2}{2} \mathbbm{1}_{n \in \Omega^{\boldsymbol{\pi},j}} \sum_{(h,u) \in \mathcal{E}_n(y)} w_{n,y}^{h,u}(t^j) \left(\tilde{F}_{\text{vr}, n}^u(j, \pi_h)\right)^2
\end{align}
Summing this over all tasks in $\mathcal{J}(y)$ produces a telescoping sum that collapses into the initial and final potentials:
\begin{align}
\ln(W_{J}) - \ln(W_0) &\le -\eta \sum_{j \in \mathcal{J}(y)} \mathbbm{1}_{n \in \Omega^{\boldsymbol{\pi},j}} \sum_{(h,u) \in \mathcal{E}_n(y)} w_{n,y}^{h,u}(t^j) \tilde{F}_{\text{vr}, n}^u(j, \pi_h) \nonumber \\
&\quad + \frac{\eta^2}{2} \sum_{j \in \mathcal{J}(y)} \mathbbm{1}_{n \in \Omega^{\boldsymbol{\pi},j}} \sum_{(h,u) \in \mathcal{E}_n(y)} w_{n,y}^{h,u}(t^j) \left(\tilde{F}_{\text{vr}, n}^u(j, \pi_h)\right)^2 \label{eq:telescoping}
\end{align}

\textbf{Step 5: Bounding the Initial and Final Potentials.} 
The initial potential is strictly the cardinality of the joint space: $W_0 = |\mathcal{E}_n(y)|$, giving $\ln(W_0) = \ln(|\mathcal{E}_n(y)|)$. The final potential $W_{J}$ is bounded below by the term for the optimal expert $(h^*, u^*)$:
\begin{align}
W_{J} \ge \exp\left(-\eta \sum_{j \in \mathcal{J}(y)} \mathbbm{1}_{n \in \Omega^{\boldsymbol{\pi},j}} \tilde{F}_{\text{vr}, n}^{u^*}(j, \pi_{h^*})\right)
\end{align}
Taking the logarithm implies:
\begin{align}
\ln(W_{J}) \ge -\eta \sum_{j \in \mathcal{J}(y)} \mathbbm{1}_{n \in \Omega^{\boldsymbol{\pi},j}} \tilde{F}_{\text{vr}, n}^{u^*}(j, \pi_{h^*})
\end{align}

\textbf{Step 6: Final Algebraic Rearrangement and Bounding the Shift.} 
Substituting the boundary bounds into Inequality (\ref{eq:telescoping}) and dividing by $\eta > 0$ isolates the difference in shifted losses:
\begin{align}
&\sum_{j \in \mathcal{J}(y)} \mathbbm{1}_{n \in \Omega^{\boldsymbol{\pi},j}} \sum_{(h,u) \in \mathcal{E}_n(y)} w_{n,y}^{h,u}(t^j) \tilde{F}_{\text{vr}, n}^u(j, \pi_h) - \sum_{j \in \mathcal{J}(y)} \mathbbm{1}_{n \in \Omega^{\boldsymbol{\pi},j}} \tilde{F}_{\text{vr}, n}^{u^*}(j, \pi_{h^*}) \nonumber \\
&\le \frac{\ln(|\mathcal{E}_n(y)|)}{\eta} + \frac{\eta}{2} \sum_{j \in \mathcal{J}(y)} \mathbbm{1}_{n \in \Omega^{\boldsymbol{\pi},j}} \sum_{(h,u) \in \mathcal{E}_n(y)} w_{n,y}^{h,u}(t^j) \left(\tilde{F}_{\text{vr}, n}^u(j, \pi_h)\right)^2
\end{align}
We substitute $\tilde{F}_{\text{vr}} = \hat{F}_{\text{vr}} + \Gamma$ back into the linear left-hand side. Because the probabilities sum to one ($\sum w = 1$), the translation term $\Upsilon_j$ strictly cancels out of the linear difference:
\begin{align}
\sum_{j \in \mathcal{J}(y)} \mathbbm{1}_{n \in \Omega^{\boldsymbol{\pi},j}} \sum_{(h,u) \in \mathcal{E}_n(y)} w_{n,y}^{h,u}(t^j) \hat{F}_{\text{vr}, n}^u(j, \pi_h) - \sum_{j \in \mathcal{J}(y)} \mathbbm{1}_{n \in \Omega^{\boldsymbol{\pi},j}} \hat{F}_{\text{vr}, n}^{u^*}(j, \pi_{h^*})
\end{align}
Taking the expectation $\mathbb{E}[\cdot]$ of this left-hand side exactly recovers the definition of $R_{n,y}(J)$.

Next, we bound the squared shifted estimator on the right-hand side using the algebraic inequality $(a+b)^2 \le 2a^2 + 2b^2$:
\begin{align}
\left(\tilde{F}_{\text{vr}, n}^u(j, \pi_h)\right)^2 &= \left(\hat{F}_{\text{vr}, n}^u(j, \pi_h) + \Upsilon_j\right)^2 \nonumber \\
&\le 2\left(\hat{F}_{\text{vr}, n}^u(j, \pi_h)\right)^2 + 2\Upsilon_j^2
\end{align}
By the definition of $\Upsilon_j$, its square is strictly bounded by the maximum squared loss among all experts at step $j$:
\begin{align}
\Upsilon_j^2 \le \max_{(h',u') \in \mathcal{E}_n(y)} \left(\hat{F}_{\text{vr}, n}^{u'}(j, \pi_{h'})\right)^2
\end{align}
Substituting this upper bound into the right-hand side of our regret inequality, the constants multiply out (since $\frac{\eta}{2} \times 2 = \eta$), yielding the final bound entirely in terms of the original estimator:
\begin{align}
R_{n,y}(J) &\le \frac{\ln(|\mathcal{E}_n(y)|)}{\eta}  
 + \eta \sum_{j \in \mathcal{J}(y)} \mathbb{E}\Bigg[ \mathbbm{1}_{n \in \Omega^{\boldsymbol{\pi},j}} \Bigg( \sum_{(h,u) \in \mathcal{E}_n(y)} w_{n,y}^{h,u}(t^j) \left(\hat{F}_{\text{vr}, n}^u(j, \pi_h)\right)^2 
+ \max_{(h',u') \in \mathcal{E}_n(y)} \left(\hat{F}_{\text{vr}, n}^{u'}(j, \pi_{h'})\right)^2 \Bigg) \Bigg].
\end{align}
Notice the summation over $j\in\mathcal{J}(y)$ is equivalent to the summation $\sum_{j=1}^\Gamma \mathbbm{1}_{y(j)=y}$.
This completes the proof.
\end{proof}

\subsection{Proof of Corollary \ref{cor:optimality}}\label{app:corollary}
\begin{proof}[Proof of Corollary \ref{cor:optimality}]
We evaluate the global time-averaged system error by analyzing the decentralized routing decisions through the Drift-Plus-Penalty framework. Let the system operate over a horizon of $T$ time slots, processing a total of $\Gamma = \sum_{t=1}^T |\mathcal{J}(t)|$ jobs, with an average arrival rate $\bar{A} = \Gamma/T$.

\textbf{Step 1: Drift-Plus-Penalty Expansion per Time Slot} \\
Let $\mathcal{L}(t) = \frac{1}{2}\sum_{n \in \mathcal{N}\setminus\mathcal{N}_1} (Q_n(t))^2$ track the global resource utilization. For a deterministic realized queue state $\mathbf{q}(t)$ at the beginning of slot $t$, the conditional Lyapunov drift $\Delta(\mathbf{q}(t)) = \mathbb{E}[\mathcal{L}(t+1) - \mathcal{L}(t) \mid \mathbf{Q}(t)=\mathbf{q}(t)]$ under the distributed policy $\boldsymbol{\pi}$ is bounded by:
$$\Delta(\mathbf{q}(t)) \le \alpha + \sum_{n \in \mathcal{N}\setminus\mathcal{N}_1} q_n(t) \mathbb{E}[C_n^{\boldsymbol{\pi}}(t) - \gamma_n\tau|\mathbf{Q}(t)=\mathbf{q}(t)]$$
To incorporate the objective, we add the expected system inference error penalty for the jobs arriving in slot $t$, scaled by the control parameter $v$. We define the per-slot unconstrained objective $\Lambda_t(\boldsymbol{\pi})$ explicitly as:
$$\Lambda_t(\boldsymbol{\pi}) = \sum_{n \in \mathcal{N}\setminus\mathcal{N}_1} Q_n(t) C_n^{\boldsymbol{\pi}}(t) + v \sum_{j \in \mathcal{J}(t)} B(j, \mathbf{x}_{n_{l(\boldsymbol{\pi},j)}^{\boldsymbol{\pi},j}}(t^j))$$
Taking the unconditional expectation over the queue states $Q(t)$ yields:
$$\mathbb{E}[\mathcal{L}(t+1) - \mathcal{L}(t)] + v \mathbb{E}\Bigg[ \sum_{j \in \mathcal{J}(t)} B(j, \mathbf{x}_{n_{l(\boldsymbol{\pi},j)}^{\boldsymbol{\pi},j}}(t^j)) \Bigg] \le \alpha + \mathbb{E}[\Lambda_t(\boldsymbol{\pi})] - \sum_{n \in \mathcal{N}\setminus\mathcal{N}_1} \mathbb{E}[Q_n(t)]\gamma_n\tau$$

\textbf{Step 2: Telescoping Sum over the Time Horizon} \\
Summing this inequality from $t=1$ to $T$ produces a telescoping sum for the drift. Because the initial queues are empty ($\mathcal{L}(1) = 0$) and the final queues are non-negative ($\mathcal{L}(T+1) \ge 0$), the cumulative expected drift is non-negative and can be omitted to maintain the upper bound:
$$v \sum_{j=1}^\Gamma \mathbb{E}[B(j, \mathbf{x}_{n_{l(\boldsymbol{\pi},j)}^{\boldsymbol{\pi},j}}(t^j))] \le \alpha + \sum_{t=1}^T \mathbb{E}[\Lambda_t(\boldsymbol{\pi})] - \sum_{t=1}^T \sum_{n \in \mathcal{N}\setminus\mathcal{N}_1} \mathbb{E}[Q_n(t)]\gamma_n\tau$$

\textbf{Step 3: Bounding with the System Regret ($R_{sys}$)} \\
We evaluate the performance of $\boldsymbol{\pi}$ against the best fixed distributed expert policy $\boldsymbol{\pi}_{best}$. The difference between their cumulative expected objectives is exactly bounded by the total system regret $R_{sys}$, which is strictly the sum of the node-level regrets across all intermediate nodes and all task types $\mathcal{Y}$:
$$R_{sys} = \sum_{n \in \mathcal{N}\setminus\mathcal{N}_K} \sum_{y \in \mathcal{Y}} R_{n,y}(\Gamma)$$
Where $R_{n,y}(\Gamma)$ is explicitly defined as:
$$R_{n, y}(\Gamma) = \mathbb{E} \Bigg[ \sum_{j=1}^\Gamma \mathbbm{1}_{y(j)=y} \mathbbm{1}_{n \in \Omega^{\boldsymbol{\pi},j}} F_n(j,\pi_n) \Bigg] - \min_{h \in \mathcal{H}_y, n' \in \mathcal{U}_{n}} \mathbb{E} \Bigg[ \sum_{j =1}^\Gamma \mathbbm{1}_{y(j)=y} \mathbbm{1}_{n \in P^{\boldsymbol{\pi}, j}} f_{n}^{n'}(j, \pi_h) \Bigg]$$
Therefore, substituting $\sum_{t=1}^T \mathbb{E}[\Lambda_t(\boldsymbol{\pi})] \le \sum_{t=1}^T \mathbb{E}[\Lambda_t(\boldsymbol{\pi}_{best})] + R_{sys}$ into our bound yields:
$$v \sum_{j=1}^\Gamma \mathbb{E}[b(j, \mathbf{x}_{n_{l(\boldsymbol{\pi},j)}^{\boldsymbol{\pi},j}}(t^j))] \le T \alpha + \sum_{t=1}^T \mathbb{E}[\Lambda_t(\boldsymbol{\pi}_{best})] - \sum_{t=1}^T \sum_{n \in \mathcal{N}\setminus\mathcal{N}_1} \mathbb{E}[Q_n(t)]\gamma_n\tau + R_{sys}$$

\textbf{Step 4: Evaluating $\boldsymbol{\pi}_{best}$ under Strict Feasibility} \\
Expanding $\Lambda_t(\boldsymbol{\pi}_{best})$, we note that because $\boldsymbol{\pi}_{best}$ is a fixed, stationary policy, its offloading actions are independent of the dynamic queue states. Thus, $\mathbb{E}[Q_n(t) C_n^{\boldsymbol{\pi}_{best}}(t)] = \mathbb{E}[Q_n(t)]\mathbb{E}[C_n^{\boldsymbol{\pi}_{best}}(t)]$.
$$\sum_{t=1}^T \mathbb{E}[\Lambda_t(\boldsymbol{\pi}_{best})] = \sum_{t=1}^T \sum_{n \in \mathcal{N}\setminus\mathcal{N}_1} \mathbb{E}[Q_n(t)]\mathbb{E}[C_n^{\boldsymbol{\pi}_{best}}(t)] + v \sum_{j=1}^\Gamma \mathbb{E}[B(j,\mathbf{x}_{n_{l(\boldsymbol{\pi}_{best},j)}^{\boldsymbol{\pi}_{best},j}}(t^j))]$$
Substituting this back, we group the queue-weighted terms:
$$\sum_{t=1}^T \sum_{n \in \mathcal{N}\setminus\mathcal{N}_1} \mathbb{E}[Q_n(t)] \Big( \mathbb{E}[C_n^{\boldsymbol{\pi}_{best}}(t)] - \gamma_n\tau \Big)$$
By definition, $\boldsymbol{\pi}_{best}$ strictly satisfies the resource constraints, ensuring $\mathbb{E}[C_n^{\boldsymbol{\pi}_{best}}(t)] \le \gamma_n\tau$. Since $\mathbb{E}[Q_n(t)] \ge 0$, this grouped term is strictly $\le 0$ and can be eliminated, perfectly isolating the accuracy penalty:
$$v \sum_{j=1}^\Gamma \mathbb{E}[B(j, \mathbf{x}_{n_{l(\boldsymbol{\pi},j)}^{\boldsymbol{\pi},j}}(t^j))] \le T B + v \sum_{j=1}^\Gamma \mathbb{E}[b(j, \mathbf{x}_{n_{l(\boldsymbol{\pi}_{best},j)}^{\boldsymbol{\pi}_{best},j}}(t^j))] + R_{sys}$$

\textbf{Step 5: Bounding the System Regret via $\Psi_{max}$} \\
We must formally define $\Psi_{max}$ as the strict maximum variance bound across any node $n$:
$$\Psi_{max} = \max_{n \in \mathcal{N}\setminus\mathcal{N}_K} \frac{1}{\Gamma} \sum_{j=1}^\Gamma \mathbb{E} \Bigg[ \mathbbm{1}_{n \in \Omega^{\boldsymbol{\pi},j}} \Bigg( \sum_{(h,u) \in \mathcal{E}_n(y(j))} w_{n,y(j)}^{h,u}(t^j) (\hat{F}_{vr,n}^u(j, \pi_h))^2 + \max_{(h',u') \in \mathcal{E}_n(y(j))} (\hat{F}_{vr,n}^{u'}(j, \pi_{h'}))^2 \Bigg) \Bigg]$$
The strict finiteness of $\Psi_{max}$ is mathematically guaranteed by the explicit uniform exploration parameter $\lambda \in (0, 1)$ integrated into the local routing distribution. Because the probability of taking any action is lower-bounded by $\frac{\lambda}{|\mathcal{U}_n| + 1}$, the marginal routing probability that any job $j$ successfully reaches the final oracle layer $K$ from node $n$ is strictly bounded away from zero, satisfying $\rho_n^\pi(j) \ge (\min_{k} \frac{\lambda}{|\mathcal{N}_{k+1}| + 1})^{K-1} > 0$. This structural bound strictly prevents the inverse probability weights within the variance-reduced estimator $(\hat{F}_{vr, n}^{u}(j, \pi_h))^2$ from diverging, thereby ensuring that the maximum expected variance $\Psi_{max}$ remains a strictly finite constant.

Summing the node-level regret bound over all task types $y \in \mathcal{Y}$ simply aggregates the job indicators $\sum_y \mathbbm{1}_{y(j)=y} = 1$, generating the bound for a single node:
$$\sum_{y \in \mathcal{Y}} R_{n,y}(\Gamma) \le \frac{|\mathcal{Y}|\ln|\mathcal{E}_{max}|}{\eta} + \eta \Gamma \Psi_{max}$$
Summing this across the $N$ intermediate nodes yields the total system regret:
$$R_{sys} \le N \Bigg( \frac{|\mathcal{Y}|\ln|\mathcal{E}_{max}|}{\eta} + \eta \Gamma \Psi_{max} \Bigg)$$
Substituting the theoretically optimal learning rate $\eta = \sqrt{\frac{|\mathcal{Y}| \ln |\mathcal{E}_{max}|}{\Gamma \Psi_{max}}}$ establishes the exact limit for $R_{sys}$:
$$R_{sys} \le 2N \sqrt{|\mathcal{Y}| \Gamma \Psi_{max} \ln|\mathcal{E}_{max}|}$$

\textbf{Step 6: Isolating the Optimality Gap} \\
We insert the upper bound for $R_{sys}$ into the inequality established in Step 4:
$$v \sum_{j=1}^\Gamma \mathbb{E}[B(j, \mathbf{x}_{n_{l(\boldsymbol{\pi},j)}^{\boldsymbol{\pi},j}}(t^j))] \le T \alpha + v \sum_{j=1}^\Gamma \mathbb{E}[B(j, \mathbf{x}_{n_{l(\boldsymbol{\pi}_{best},j)}^{\boldsymbol{\pi}_{best},j}}(t^j)))] + 2N \sqrt{|\mathcal{Y}| \Gamma \Psi_{max} \ln|\mathcal{E}_{max}|}$$
To convert this to time-averaged errors, we divide the entire inequality by $v \Gamma$. Recognizing that the arrival rate maps to $\frac{T}{\Gamma} = \frac{1}{\bar{A}}$, the first term simplifies to $\frac{\alpha}{v\bar{A}}$. 
Applying the discretization assumption $\Phi_J(\boldsymbol{\pi}_{best}) \le \Phi_J(\boldsymbol{\pi}^*) + \epsilon_0$ and solving for the optimality gap yields the final strict bound:
$$\Phi_\Gamma(\boldsymbol{\pi}) - \Phi_\Gamma(\boldsymbol{\pi}^*) \le \frac{\alpha}{v \bar{A}} + \frac{2 N}{v} \sqrt{\frac{|\mathcal{Y}| \Psi_{max} \ln|\mathcal{E}_{max}|}{\Gamma}} + \epsilon_0$$
This completes the proof.
\end{proof}

\subsection{Proof of Submodularity}\label{app:submodular}
\begin{proposition}[Diminishing Returns]
For each node $n$, the model-placement utility function $U_n(\cdot)$ defined in Eq.~\eqref{eq:utility} is submodular. Consequently, its marginal gain $\Delta U_n(m' \mid \mathcal{S})$ exhibits diminishing returns with respect to the placement set $\mathcal{S} \subseteq \mathcal{M}$.
\end{proposition}
% To justify the use of the Greedy algorithm, which guarantees a $(1-1/e)$-approximation for submodular maximization subject to a knapsack constraint \citep{sviridenko2004note}, we rigorously prove that the marginal gain function is monotone and submodular.

\begin{proof}
Recall that the utility of a placement $\mathcal{S} \subseteq \mathcal{M}$ at node $n$ is
\[
U_n(\mathcal{S})
= 
\mathbb{E}_{J \sim \mathcal{I}_n(t)}
\!\left[
1 - \min_{m \in \mathcal{S}} \epsilon(J,m,\zeta)
\right]
- 
\nu \sum_{m \in \mathcal{S}} s_m \mathbbm{1}_{m \notin M_n(t-1)} .
\]

The marginal gain of adding a model $m'$ to a placement set $\mathcal{S}$ is
\[
\Delta U_n(m' \mid \mathcal{S})
=
U_n(\mathcal{S} \cup \{m'\}) - U_n(\mathcal{S}),
\]
which, using Eq.~(16), can be written as
\[
\Delta U_n(m' \mid \mathcal{S})
=
\mathbb{E}_{J \sim \mathcal{I}_n(t)}
\!\left[
\min_{m \in \mathcal{S}} \epsilon(J,m,\zeta)
-
\min_{m \in \mathcal{S} \cup \{m'\}} \epsilon(J,m,\zeta)
\right]
-
\nu s_{m'} \mathbbm{1}_{m' \notin M_n(t-1)} .
\]

Let $\mathcal{A} \subseteq \mathcal{B} \subseteq \mathcal{M}$. For any job $J$,
\[
\min_{m \in \mathcal{B}} \epsilon(J,m,\zeta)
\le
\min_{m \in \mathcal{A}} \epsilon(J,m,\zeta).
\]
Define
\[
g(x) = \max(0, x - \epsilon(J,m',\zeta)),
\]
which is monotone non-decreasing in $x$. Applying $g(\cdot)$ to both sides yields
\[
\max\!\left(0,\min_{m \in \mathcal{B}}\epsilon(J,m,\zeta) - \epsilon(J,m',\zeta)\right)
\le
\max\!\left(0,\min_{m \in \mathcal{A}}\epsilon(J,m,\zeta) - \epsilon(J,m',\zeta)\right).
\]
Taking expectation over $J \sim \mathcal{I}_n(t)$ preserves the inequality:
\[
\Delta U_n^{err}(m' \mid \mathcal{B})
\le
\Delta U_n^{err}(m' \mid \mathcal{A}).
\]
Thus, the error-reduction component is submodular and monotone.

The switching penalty
\[
P_n(m') = \nu s_{m'} \mathbbm{1}_{m' \notin M_n(t-1)}
\]
depends only on the historical placement $M_n(t-1)$ and the candidate model $m'$, and is independent of the current constructed set $\mathcal{S}$. Therefore, it is a modular function.

Since the sum of a submodular function and a modular function remains submodular, the total marginal gain $\Delta U_n(m' \mid \mathcal{S})$ is submodular. Hence, $U_n(\cdot)$ satisfies the diminishing returns property.

\end{proof}

Because $U_n(\cdot)$ is submodular but not necessarily monotone due to the switching penalty, we adopt a marginal-density greedy rule with early stopping when no candidate yields positive marginal gain. While the classical $(1-1/e)$ approximation guarantee holds for monotone submodular maximization under a knapsack constraint~\citep{sviridenko2004note}, our formulation preserves submodularity and ensures diminishing returns and stable incremental improvement under memory capacity constraints.

\begin{table*}[htbp]
\centering
\caption{Summary of notation}
\label{tab:notation-main}
\renewcommand{\arraystretch}{1.1}
\resizebox{0.75\linewidth}{!}{
\begin{tabular}{lp{12.5cm}}
\toprule
\textbf{Symbol} & \textbf{Definition} \\
\midrule
\multicolumn{2}{l}{\textit{System and Network Model}} \\
\midrule
$\mathcal{N}$ & Set of computing nodes. \\
$K$ & Number of non-empty hierarchical layers. \\
$\mathcal{N}_{k}$ & Set of nodes in layer $k$. \\
$t, \tau$ & Time slot index and the duration of each time slot $t$, respectively. \\
$\mathcal{J}(t)$ & Set of jobs received in time slot $t$. \\
$j, t^{j}$ & A specific job index and its corresponding arrival slot. \\
$\mathcal{I}_n(t)$ & Distribution of jobs arriving at node $n$ during time slot $t$. \\
$n_{k}^{\boldsymbol{\pi},j}$ & Node in layer $k$ visited by job $j$. \\
$y(j), \mathcal{Y}$ & Task type of job $j$, and the set of all possible task types. \\
\midrule
\multicolumn{2}{l}{\textit{Routing and Inference Actions}} \\
\midrule
$Z_{n}(j), z_{n}(j)$ & Random confidence variable observed at node $n$ for job $j$, and its realized deterministic score. \\
$\mathcal{A}_{k}$ & Action set of a node in layer $k$, where $\mathcal{A}_{k} \triangleq \{0\} \cup \mathcal{N}_{k+1}$. \\
$\pi_{n}(\cdot), \boldsymbol{\pi}$ & Randomized offloading policy at node $n$, and the distributed policy across all nodes $\boldsymbol{\pi} = \{\pi_{n} : n \in \mathcal{N}\}$. \\
$\mathbf{O}_{n}^{\pi_{n},j}, \mathbf{o}_{n}^{\pi_{n},j}$ & Random one-hot action vector generated by $\pi_{n}(\cdot)$, and its realized deterministic vector. \\
$\Omega^{\boldsymbol{\pi},j}, \omega^{\boldsymbol{\pi},j}$ & Random inference path induced by policy $\pi$, and its deterministic realized path. \\
$l(\boldsymbol{\pi},j)$ & Realized final layer where job $j$ is finalized. \\
\midrule
\multicolumn{2}{l}{\textit{Memory and Model Placement}} \\
\midrule
$\mathcal{M}$ & Global set of available candidate models. \\
$\mathcal{M}_{n}(t)$ & Set of models loaded at node $n$ during slot $t$. \\
$\mathbf{x}_{n}(t), x_{n}^{m}(t)$ & Binary vector for the onloading decision at node $n$ in slot $t$, and the indicator for model $m$. \\
$\mu_{n}, s_{m}$ & Memory capacity constraint of node $n$, and the memory footprint of model $m$. \\
$D$ & Number of time slots between periodic model onloading updates. \\
$U_n(\mathcal{M}_{n}(t))$ & Total utility function for evaluating model placement. \\
$\nu$ & Switching-cost penalty parameter. \\
\midrule
\multicolumn{2}{l}{\textit{Costs, Errors, and Queues}} \\
\midrule
$B(j,\mathbf{x}_{n}), b(j,\mathbf{x}_{n})$ & Random inference error of executing job $j$ with placement $\mathbf{x}_{n}$, and its realized deterministic error. \\
$\epsilon(j,m,\zeta)$ & Realized deterministic inference error of processing job $j$ using model $m$. $\zeta$ denotes the randomness from the inference process itself. \\
$c^{j}(n^{\prime},n)$ & Deterministic transmission cost for offloading job $j$ from node $n^{\prime}$ to node $n$. \\
$C_{n}^{\boldsymbol{\pi}}(t)$ & Total random communication cost incurred by node $n$ in slot $t$. \\
$\gamma_{n}$ & Maximum allowable long-term average resource consumption rate for node $n$. \\
$Q_{n}(t), q_{n}(t)$ & Random virtual queue tracking resource deviation at node $n$, and its deterministic realized state. \\
$\mathbf{Q}(t), \mathbf{q}(t)$ & Random vector of all virtual queues, and its deterministic realized vector. \\
$\mathcal{L}(t), \Delta(\mathbf{q}(t))$ & Lyapunov function tracking network congestion, and the conditional Lyapunov drift. \\
$v$ & Control parameter balancing performance optimality against constraint satisfaction. \\
\midrule
\multicolumn{2}{l}{\textit{Contextual Bandits and Loss Estimation}} \\
\midrule
$\mathcal{U}_{n}$ & Set of valid upstream destinations for intermediate node $n$. \\
$\mathcal{H}_{y}, \theta_{h}$ & Set of experts for task type $y$, and the deterministic offloading threshold associated with expert $h$. \\
$\mathcal{E}_{n}(y), e$ & Joint expert space for task type $y$ at node $n$ ($\mathcal{H}_{y} \times \mathcal{U}_{n}$), and a joint expert $e = (h,n^{\prime})$. \\
$w_{n,y}^{e}(t)$ & Probability weight for joint expert $e$ at node $n$ for task type $y$ in slot $t$. \\
$p_{n}^{j}(a)$ & Probability of offloading job $j$ to action $a$. \\
$F_{n}(j,\pi_{n}), \bar{f}_{n}(j,\pi_{n})$ & Random per-job loss at node $n$ under $\pi_{n}$, and the deterministic expected loss. \\
$f_{n}^{n^{\prime}}(j,\pi_{h})$ & True full-feedback loss evaluated under deterministic threshold rule $\pi_{h}$. \\
$\hat{g}_{n,y}^{(h,n^{\prime})}(t-1)$ & Deterministic, realized cumulative loss up to slot $t-1$. \\
$\mathbbm{1}_{\text{fb},n}(j), \rho_{n}^{\boldsymbol{\pi}}(j)$ & Random feedback indicator, and the deterministic marginal probability job $j$ reaches the oracle. \\
$\hat{F}_{\text{naive},n}^{n'}(j,\pi_h), \hat{f}_{\text{naive},n}^{n'}(j,\pi_h)$ & Random naive unbiased estimator of the expert loss, and its realized value. \\
$\hat{F}_{\text{vr}, n}^{n'}(j,\pi_h), \hat{f}_{\text{vr}, n}^{n'}(j,\pi_h)$ & Random variance-reduced expert loss estimator, and its realized value. \\
$\bar{f}_{n,y}^{n^{\prime}}(\pi_{h})$ & Deterministic baseline term (unbiased estimate of conditional theoretical loss) for task type $y$. \\
$R_{n,y}(\Gamma)$ & Expected regret of node $n$ for task type $y$ relative to the best fixed policy over horizon $\Gamma$. \\
\bottomrule
\end{tabular}}
\end{table*}

\end{document}